%% file: main.tex
\documentclass{article} % For LaTeX2e
\usepackage[T1]{fontenc}
\usepackage{iclr2027_conference,times}

\input{math_commands.tex}

\usepackage{url}
\usepackage{amsthm}
\usepackage{graphicx}
\usepackage{booktabs}
\usepackage{array}
\usepackage{xcolor}
\usepackage{colortbl}
\usepackage{algorithm}
\usepackage{algpseudocode}
\usepackage{tabularx}
\usepackage{capt-of}
\usepackage{tcolorbox}
\usepackage{amssymb}
\usepackage{wrapfig}
\usepackage{hyperref}
\tcbuselibrary{skins,breakable}
\newcommand{\secondbest}[1]{\underline{#1}}
\newtheorem{proposition}{Proposition}

\title{Learning to Accumulate Knowledge with\\ Mutual Information}

\author{
\makebox[\textwidth][c]{%
Yuyang Zhao$^{1}$ \quad
Lizi Liao$^{2}$ \quad
Leyang Shen$^{3}$ \quad
\textbf{Xiaoyan Zhao}$^{3}$ \quad
Yang Zhang$^{3}$
}\\
\makebox[\textwidth][c]{%
\textbf{Fuli Feng$^{1}$} \quad
\textbf{Xiangnan He}$^{1}$
}\\[2pt]
\makebox[\textwidth][c]{%
$^{1}$University of Science and Technology of China \quad
$^{2}$Singapore Management University
}\\
\makebox[\textwidth][c]{%
$^{3}$National University of Singapore
}\\
}
\iclrfinalcopy % Uncomment for camera-ready version, but NOT for submission.
\begin{document}

\maketitle
\lhead{Preprint. Under review.} 

\begin{abstract}
\input{sections/abstract}
\end{abstract}

\input{sections/intro}

\input{sections/related_work}

\input{sections/method}

\input{sections/experiments}

\input{sections/conclusion}

% \section*{Limitations}
% Our evaluation focuses on two text-based environments and knowledge banks constructed from bounded trajectory pools. Although we evaluate transfer across frozen executors, these experiments do not establish generalization to multimodal tasks or continually changing environments. Retrieval sweeps assess how existing banks support execution as more entries are supplied and they do not measure sustained knowledge accumulation over time.

\subsection*{AI use statement}

In this work, we used generative AI tools to assist with the implementation of software code and the analysis and interpretation of experimental results. We did not use generative AI tools to formulate the core research ideas, hypotheses, or methodological design of this work.

Additionally, we used generative AI tools to improve the readability and language of the manuscript and to create or refine minor visual elements in some figures.

All AI-assisted code was reviewed and tested by the authors, and all AI-assisted analyses of experimental results were independently verified against the original experimental outputs. All AI-assisted text and visual elements were reviewed and revised by the authors. We take responsibility for the final content of this work, including text, claims, code, analyses, and artifacts produced with the aid of generative AI.

% (This section is \textbf{required} and does not count toward the page limit.)

% In this work, we used generative AI tools for [tasks with required disclosure].
% We have not used generative AI tools for [other tasks with required disclosure],
% and [the rest of the required disclosure tasks] are not applicable to this work.
% Additionally, we used generative AI tools for [tasks with recommended
% disclosure]. We have reviewed all AI-assisted work. [Elaborate. For example, “we
% checked LLM-generated research ideas for potential plagiarism through a manual
% literature survey”, “LLM-generated code was verified and tested for correctness
% by 2 authors”, etc.]. We take responsibility for the final content of this work,
% including text, claims or artifacts produced with the aid of generative AI.

% See the ICLR 2027 AI Policy for Authors for more details. This statement should
% not be more than 1 page.

\subsection*{Reproducibility statement}
Section~\ref{sec:method} defines the feedback signals and training objective. Section~\ref{sec:experiment-setup} describes the ALFWorld and WebShop evaluation protocol. Appendix~\ref{app:agent-prompts} provides the prompts used for knowledge generation and execution. Appendix~\ref{app:experiment-details} reports the hyperparameter settings in Table~\ref{tab:training_hyperparameters}. Appendix~\ref{app:training-algorithm} presents the training pseudocode and GRPO update. Appendix~\ref{app:baseline-reproduction} explains how the baselines were implemented and adapted to the shared evaluation protocol.

\subsection*{Ethics statement}
Human volunteers wrote knowledge entries from sampled agent trajectories to construct the Human-written baseline. Participation was voluntary, and informed consent was obtained from all volunteers.


\bibliography{ref}
\bibliographystyle{iclr2027_conference}

\newpage
\appendix
\input{sections/appendix/proofs}
\input{sections/appendix/case_study}
\input{sections/appendix/prompt}
\input{sections/appendix/baseline_reproduction}
\input{sections/appendix/experiment_details}
\input{sections/appendix/pseudo_code}

\end{document}

%% file: math_commands.tex
\usepackage{amsmath,amsfonts,bm}

\def\eqref#1{equation~\ref{#1}}
\def\1{\bm{1}}

\DeclareMathAlphabet{\mathsfit}{\encodingdefault}{\sfdefault}{m}{sl}
\SetMathAlphabet{\mathsfit}{bold}{\encodingdefault}{\sfdefault}{bx}{n}

%% file: sections/abstract.tex
% Large language model (LLM) agents can improve their performance by reusing knowledge distilled from past interactions. However, effective knowledge consumption does not ensure reliable curation, and even strong LLMs can generate guidance with limited downstream utility. Training a curator on task success alone does not explicitly reward preserving the distinct information contained in different experiences. We propose \textbf{Knowledge Weaver}, a reinforcement learning framework that uses mutual information (MI) to guide a compact language model in accumulating reusable knowledge from agent trajectories. It trains the knowledge curator with an MI-based objective to preserve information from distinct agent experiences during knowledge generation. The training objective combines token-level MI rewards with task success and marginal gains from combining knowledge entries. These signals encourage the curator to preserve experience-dependent information while generating knowledge that is individually useful and complementary to other entries. On ALFWorld and WebShop, Knowledge Weaver achieves success rates of 54.0\% and 42.0\% at $k=10$, improving over GRPO by 16.9 and 18.7 percentage points, with the executor frozen. Its knowledge banks also outperform the evaluated prompt-based, established, and human-written banks in overall ALFWorld success rate and WebShop score.

Large language model (LLM) agents can improve their performance by reusing knowledge distilled from past interactions. However, curating new experiences into a knowledge bank that becomes more useful as it grows remains challenging. Effective knowledge accumulation should limit redundant overlap among entries and ensure that new knowledge contributes beyond what the bank already provides. Yet training a curator with Group Relative Policy Optimization (GRPO) on standalone task success can reinforce general guidance even when it duplicates existing knowledge. Therefore, we propose \textbf{Knowledge Weaver}, a reinforcement learning framework that trains a language model to curate reusable knowledge from agent trajectories. We couple feedback inspired by token-wise mutual information (MI) with marginal success rewards to guide knowledge accumulation. Together, these signals encourage the curator to preserve distinct information from experience and produce entries that improve task success when added to existing knowledge. Standalone success rewards also favor entries that are useful on their own. On ALFWorld and WebShop, Knowledge Weaver achieves mean success rates of 54.0\% and 42.0\% with $k=10$ retrieved entries, exceeding GRPO by 16.9 and 18.7 percentage points, respectively. Its knowledge banks also outperform the evaluated prompt-based and established banks, including human-written banks, in overall ALFWorld success rate and WebShop score with the executor frozen. Our codebase is available at \href{https://github.com/LaoKuiZe/Knowledge-Weaver}{https://github.com/LaoKuiZe/Knowledge-Weaver}.

%% file: sections/intro.tex
\section{Introduction}
\label{sec:introduction}

% Large language model (LLM) agents increasingly tackle complex tasks through multi-step reasoning and interaction with external environments \citep{yao2022react,gur2024real,yang2024swe}. Their performance can improve substantially when supplied with high-quality, task-specific knowledge \citep{wang2024agent,li2026skillsbench,ouyang2026reasoningbank}. 
Large language model (LLM) agents interact with external environments to solve complex tasks \citep{yao2022react,gur2024real,yang2024swe,carl}. Reusing knowledge from these interactions offers a route to agent self-improvement, allowing experience from previous tasks to inform future decisions. Prior work shows that high-quality, task-specific knowledge can improve agent performance \citep{wang2024agent,li2026skillsbench,ouyang2026reasoningbank}.
Such knowledge includes reusable workflows and executable skills that guide action sequences \citep{wang2023voyager,wang2024agent,yuan2024craft}, as well as distilled lessons about environmental constraints and effective strategies \citep{majumder2023clin,chen2024automanual}. Stored in external knowledge banks, relevant entries can be retrieved and incorporated into the agent's context without updating the executor's parameters \citep{zheng2024synapse,fu2024autoguide}. This makes knowledge accumulation a means of retaining experience beyond individual interactions and supporting its reuse across tasks.

\begin{figure}[t]
    \centering
    \includegraphics[width=0.88\linewidth]{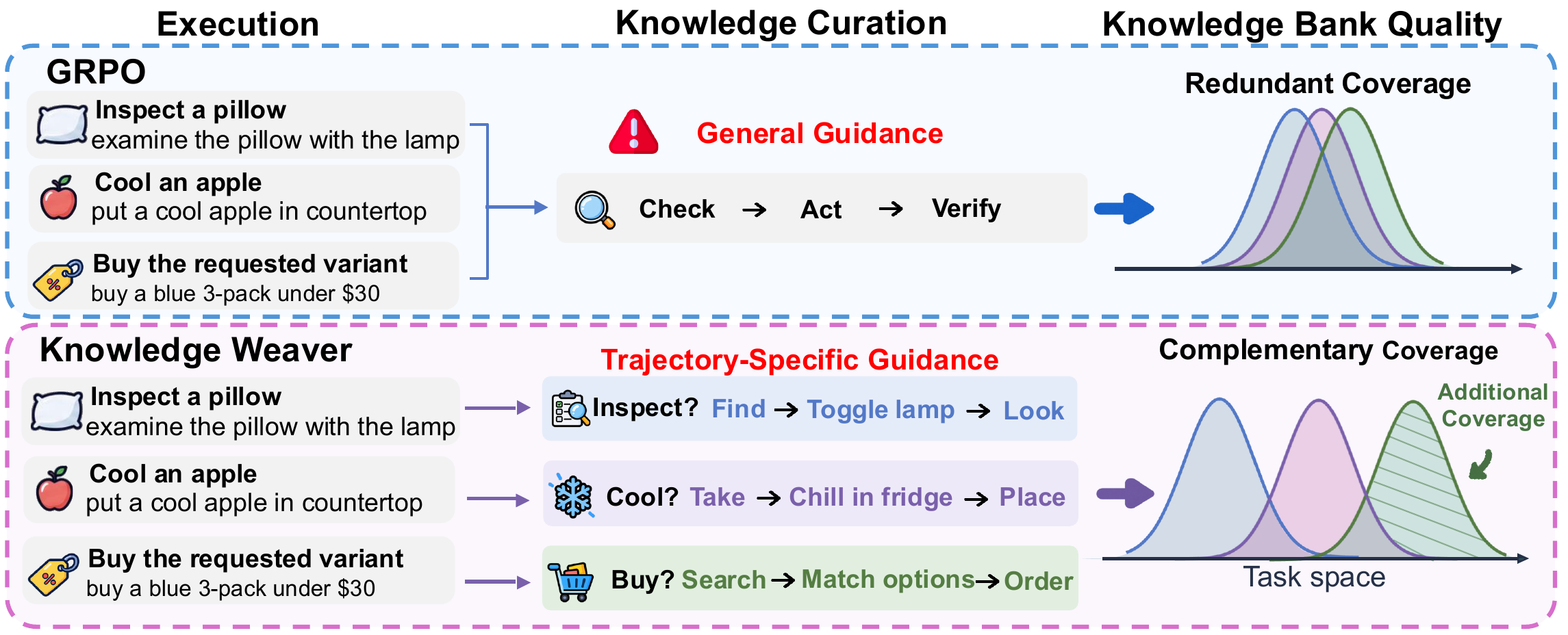}
    \vspace{-0.2cm}
    \caption{Conceptual illustration of general guidance and complementary knowledge. General guidance can remain useful for individual tasks while providing limited additional information across entries. Knowledge Weaver combines MI-inspired feedback with execution rewards to encourage experience-dependent knowledge and incremental utility.}
    \vspace{-0.3cm}
    \label{fig:intro}
\end{figure} 

While these benefits establish the value of using knowledge, producing it reliably remains a separate challenge. Existing work relies on pipeline design or training-based methods for knowledge curation. Prompt-based methods extract reflective lessons, conditional guidelines, and reusable workflows from trajectories \citep{shinn2023reflexion,zhao2024expel,fu2024autoguide,wang2024agent}, then consolidate them through memory or playbook revision \citep{zhang2026agentic,ouyang2026reasoningbank,suzgun2026dynamic}. 
% Reinforcement learning further introduces execution feedback to internalize supplied skills or jointly improve task solving and knowledge curation \citep{lu2026skill0,xia2026skillrl,shi2026skill1,yao2026skillrise}. However, task success alone does not explicitly reward preserving the information that distinguishes different experiences. A curator can therefore produce guidance that helps individual tasks while repeatedly discarding the conditions that made those experiences informative.
Training-based methods integrate knowledge into policy improvement through skill internalization \citep{lu2026skill0} or the joint evolution of task execution and reusable skills \citep{xia2026skillrl,shi2026skill1,yao2026skillrise}. When curation is rewarded through downstream task success, this feedback assesses whether generated knowledge helps execution. 
% 找position，
However, these advances in knowledge construction and use leave a distinct question of how to train a curator to produce better knowledge for a fixed executor. Addressing this question requires considering both the utility of individual entries and what they contribute as knowledge accumulates.

% We propose \textbf{Knowledge Weaver}, an RL framework that trains a language model as a knowledge curator to extract reusable knowledge from agent trajectories. We use mutual information (MI), which measures statistical dependence, to motivate preserving distinctions between source experiences in generated knowledge \citep{cover1991elements}. This makes experience dependence an explicit training target, addressing information loss that execution outcomes alone may not reveal.
% 什么是好的curator， 怎么组成，什么是好的Knowledge Accumulation。减小overlapping，增大incremental

Building on this, we propose \textbf{Knowledge Weaver}, an RL framework that trains a language model as a knowledge curator to extract reusable knowledge from agent trajectories. A good curator should produce knowledge that improves task performance and adds useful information as the bank grows. Effective accumulation therefore calls for both limited redundant overlap and substantial incremental utility. Figure~\ref{fig:intro} illustrates these considerations through knowledge coverage over a task space. Optimizing only for standalone task success can favor general guidance that repeatedly covers the same tasks, yielding substantial overlap and limited additional coverage as the bank grows. Complementary entries extend the bank's coverage, with the shaded region indicating the incremental contribution of a new entry. We therefore approach knowledge accumulation through both the distinct information retained across entries and the additional utility each entry provides.

% Effective knowledge accumulation should improve task performance while increasing the useful information retained as the bank grows, with limited redundant overlap. We therefore reward individual task success and address redundancy through two complementary signals. Marginal success gains assess whether an entry contributes beyond existing knowledge through execution outcomes. Mutual information (MI), which measures statistical dependence \citep{cover1991elements}, motivates preserving distinct information from source experiences during knowledge generation.

Knowledge Weaver couples token-level feedback inspired by mutual information (MI) with marginal success rewards to preserve distinct information and encourage incremental contributions to the bank. The MI-inspired signal contrasts generation probabilities across trajectory contexts, assigning greater credit to context-sensitive tokens than to tokens whose probabilities remain similar across experiences. To reward incremental contributions, we measure the gain in task success when a new entry is added to retrieved knowledge. This marginal reward favors entries that improve execution beyond what the bank already provides. We also reward task success with each entry alone to encourage individual utility. These signals jointly train the curator, whose generated knowledge is retrieved to guide a frozen executor.

% We evaluate Knowledge Weaver on ALFWorld\citep{shridhar2020alfworld} and WebShop\citep{yao2022webshop}. With $k=10$ retrieved entries, Knowledge Weaver achieves success rates of 54.0\% and 42.0\%, exceeding GRPO by 16.9 and 18.7 percentage points, respectively. Its knowledge banks also outperform the evaluated prompt-based and established banks, including human-written banks, on overall ALFWorld success rate and both WebShop metrics. These improvements are obtained with the same frozen executor across comparisons, demonstrating that training knowledge curation can improve downstream execution without changing the executor's parameters.
We evaluate whether learned knowledge improves a frozen executor on ALFWorld, a household task environment \citep{shridhar2020alfworld}, and WebShop, an online shopping environment \citep{yao2022webshop}. Before each task, we retrieve $k=10$ knowledge entries from the bank using the task query and provide them to the frozen executor as guidance. Knowledge Weaver achieves success rates of 54.0\% and 42.0\%, respectively. These results exceed task-success-only GRPO by 16.9 and 18.7 percentage points under the same executor. These gains demonstrate improved task performance and support a curation strategy aimed at reducing redundant overlap among knowledge entries.
% The gains across both environments support training knowledge curation to improve downstream task performance without updating the executor's parameters.

Our contributions are summarized as follows.
\begin{itemize}
    \item \textbf{Redundancy in knowledge curation.}
    We identify a limitation of GRPO using only task success as reward, which can produce general guidance with substantial informational redundancy across knowledge entries.

    \item \textbf{MI-guided curator training.}
    We propose Knowledge Weaver, which combines token-level information feedback with execution rewards to train useful and complementary knowledge generation.

    \item \textbf{Improved downstream performance.}
    On ALFWorld and WebShop, Knowledge Weaver improves SR over task-reward-only GRPO by 16.9 and 18.7 percentage points at $k=10$, respectively. Its banks also exceed human-written banks on the reported
    aggregate metrics.
\end{itemize}

%% file: sections/related_work.tex
\section{Related Work}
\label{sec:related-work}

\vspace{-2mm}
\paragraph{Knowledge Construction from Experience.} LLM agents construct reusable knowledge through memory organization, experience distillation, and skill synthesis. Memory systems extract and organize information for later retrieval \citep{chhikara2025mem0,xu2026mem,zhang2026nextmem}, while experience distillation produces reflective lessons and generalizable strategies \citep{shinn2023reflexion,zhao2024expel,majumder2023clin,ouyang2026reasoningbank}. Structured representations include conditional guidelines \citep{fu2024autoguide,chen2024automanual}, reusable workflows \citep{wang2024agent}, and executable skills \citep{wang2023voyager,yuan2024craft,zheng2025skillweaver}. AutoSkill \citep{yang2026autoskill} maintains skills derived from interaction traces, while Trace2Skill \citep{ni2026trace2skill} consolidates trajectory-level lessons into transferable guidance. Skill-Pro \citep{mi2026skill} uses execution feedback to generate and verify procedural skills without parameter updates. RL-based approaches further learn to use and construct reusable knowledge. Skill0 \citep{lu2026skill0} trains agents to internalize supplied skills, while Skill0.5 \citep{zhu2026skill0} combines general skill internalization with task-specific skill utilization. SkillRL \citep{xia2026skillrl} couples skill distillation and library evolution with executor training. Skill1 \citep{shi2026skill1} jointly learns skill selection, utilization, and distillation, while SkillRise \citep{yao2026skillrise} assigns curation credit through subsequent execution outcomes. Knowledge Weaver complements execution feedback with explicit information feedback to train knowledge generation while keeping the executor fixed.

\vspace{-2mm}
\paragraph{Mutual Information.}
Mutual information (MI) provides a general criterion for learning informative representations and encouraging dependence between related variables. In representation learning, Contrastive Predictive Coding captures information predictive of future observations, while Deep InfoMax promotes dependence between inputs and learned representations \citep{oord2018representation,hjelm2018learning}. InfoGAN uses MI between latent codes and generated observations to encourage interpretable factors of variation \citep{chen2016infogan}. In reinforcement learning, DIAYN maximizes MI between skill identities and visited states to discover distinguishable behaviors \citep{eysenbach2018diversity}. In dialogue generation, MI-based objectives discourage generic responses and favor responses that convey information about the conversational context \citep{li2016diversity,zhang2018generating}. AMI further models source-target dependence through adversarial optimization of forward and backward generation networks \citep{pan2020adversarial}. For LLM agents, RAGEN-2 introduces MI-inspired proxies to assess the dependence of generated reasoning on input prompts \citep{wang2026ragen}. In knowledge curation, input dependence concerns whether generated entries retain the distinct information carried by their source trajectories. Knowledge Weaver applies token-level MI-inspired feedback during curator training to encourage experience-specific content.

%% file: sections/method.tex
\section{Method}
\label{sec:method}

\begin{figure}[t]
\centering
\includegraphics[width=\linewidth]{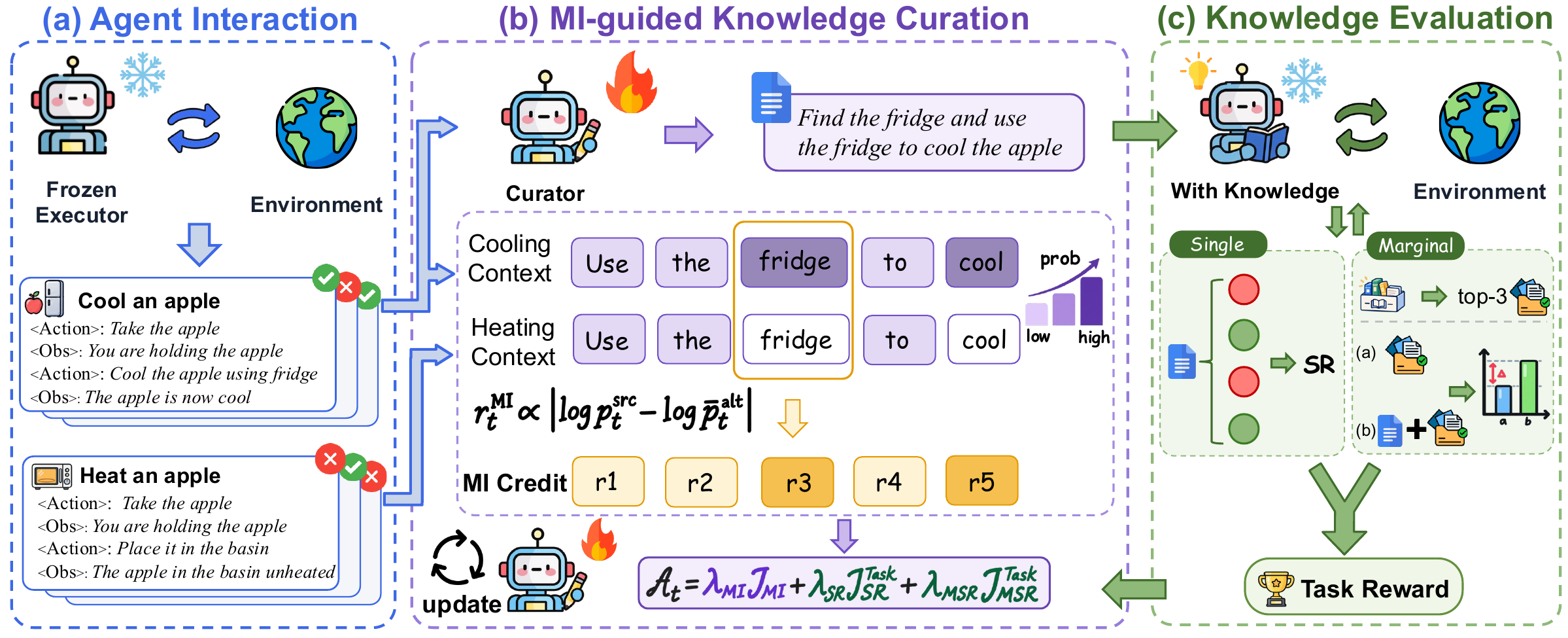}
\vspace{-0.7cm}
\caption{Overview of Knowledge Weaver. Token-level MI feedback promotes experience-dependent knowledge, while execution rewards assess individual utility and marginal gains over retrieved knowledge. Only the curator is trained and the executor remains frozen.}    \vspace{-0.5cm}
\label{fig:framework}
\end{figure}

\subsection{Preliminaries}
\label{sec:preliminaries}
As shown in Figure~\ref{fig:framework}, a trainable curator $\pi_\theta$ generates a knowledge entry $z\sim\pi_\theta(\cdot\mid x)$ from a context $x$ of trajectories sampled from a pool $\mathcal{T}$, each recording a task's observations and actions. Training maintains an initially empty bank $\mathcal{B}$, expanded between updates. After training, the curator generates a separate evaluation bank $\mathcal{C}=\{z^{(i)}\}_{i=1}^{K}$ of valid entries from sampled contexts. A frozen executor uses a knowledge set $S$ to solve task $q$. Writing its binary success as $s(q,S)$, we measure empirical utility on a finite reward-task set $\mathcal{D}$ by
\begin{equation}
U(S)=\frac{1}{|\mathcal{D}|}\sum_{q\in\mathcal{D}}s(q,S).
\label{eq:empirical-utility}
\end{equation}
During evaluation, $S=\mathcal{R}_k(q,\mathcal{C})$ contains at most $k$ entries retrieved once from the initial query and held fixed throughout execution.

\subsection{Effective Knowledge Accumulation}
\label{sec:coverage-model}
Effective accumulation requires both individual utility and contributions beyond existing knowledge. With two equally weighted task types, a routine for an already-supported type and one for an uncovered type can be equally useful alone, yet only the latter extends the bank's capabilities.

We formalize this distinction with a nonnegative, integrable utility profile $f_i(u)$ for each entry $z^{(i)}$ over an idealized task coordinate $u\in\mathbb{R}$. The profile describes where and how strongly an entry helps; its area represents total individual utility. To avoid counting the same available utility repeatedly, bank coverage takes the best profile at each location,
% 这里之前需要更多的assumption
\begin{equation}
\mathcal{E}(\mathcal{C})=\int_{\mathbb{R}}\max_{1\le i\le K}f_i(u)\,du.
\label{eq:knowledge-quality}
\end{equation}
Thus, repeated profiles add no coverage. This model assumes ideal selection of an entry without retrieval errors or interactions between entries; it illustrates utility overlap without predicting execution success. For a tractable comparison, use Gaussian profiles with fixed centers $c_1<\cdots<c_K$, $K\ge2$, common height $h>0$, and width $\sigma>0$,
\begin{equation}
f_i(u;h,\sigma)=h\exp\!\left(-\frac{(u-c_i)^2}{2\sigma^2}\right).
\label{eq:utility-profile}
\end{equation}
The center locates peak utility, while height and width describe its strength and extent. Each entry has total utility $A=\sqrt{2\pi}\sigma h$; write the resulting coverage as $\mathcal{E}(h,\sigma)$.
\begin{proposition}[Utility scaling]
\label{prop:height}
For fixed centers and width, $\mathcal{E}(h,\sigma)=h\mathcal{E}(1,\sigma)$ is strictly increasing in $h>0$.
\end{proposition}
To compare overlap at equal individual utility, we instead fix each profile's area, so height must decrease as width grows.
\begin{proposition}[Coverage under a fixed utility budget]
\label{prop:width}
For fixed centers and per-entry utility $A>0$, let $h_A(\sigma)=A/(\sqrt{2\pi}\sigma)$. Then $\mathcal{E}_A(\sigma)=\mathcal{E}(h_A(\sigma),\sigma)$ strictly decreases with $\sigma>0$.
\end{proposition}
The first result isolates stronger individual utility. In the second, total individual utility remains $KA$, but broader profiles overlap more and yield less bank coverage. Narrowing profiles at fixed height need not improve coverage. A new entry adds coverage only where its utility exceeds the existing envelope (Figure~\ref{fig:intro}). Appendix~\ref{app:proofs} gives both proofs and the marginal-coverage identity.

\paragraph{Why standalone-success GRPO fails.}
Group Relative Policy Optimization (GRPO) \citep{shao2024deepseekmath} samples $G\ge2$ candidates $z^{(g)}$ from a curator snapshot $\pi_{\mathrm{old}}(\cdot\mid x)$ for the same source context $x$. Under outcome supervision, every token $t$ in candidate $g$ shares an advantage, the relative learning signal for the policy update,
\begin{equation}
A^{\mathrm{GRPO}}_{g,t}=\frac{R_g-\bar R}{\sigma_R+\epsilon_{\mathrm{num}}},
\label{eq:grpo-advantage}
\end{equation}
where $R_g$ is the candidate reward, $\bar R$ and $\sigma_R$ are its group's reward mean and standard deviation, and $\epsilon_{\mathrm{num}}>0$ ensures numerical stability.

With standalone success $R_g=U(\{z^{(g)}\})$, this advantage contains no bank information. An already-covered routine receives positive advantage when $R_g>\bar R$, with no preference for an equally successful complementary entry. Shared outcome credit also treats source-specific and generic tokens alike. We therefore introduce bank-relative gains to reward contributions beyond existing knowledge and source-dependent credit to encourage retention of the distinct information in each experience. These goals motivate two complementary signals. Marginal success feedback rewards contributions beyond existing knowledge (Section~\ref{sec:mi-training}), while MI-inspired token credit encourages retention of source-dependent information (Section~\ref{sec:token-mi}).

% 这里公式名字要改
\subsection{Marginal Success Feedback}
\label{sec:mi-training}
Marginal success-rate (MSR) feedback encourages additions that improve execution beyond what the bank already supports. For each task $q\in\mathcal{D}$, retrieve up to $m$ entries $B_q=\mathcal{R}_m(q,\mathcal{B})$ from a fixed training-bank snapshot and compute
\begin{equation}
\mathrm{MSR}(z)
=\frac{1}{|\mathcal{D}|}\sum_{q\in\mathcal{D}}
\left[s(q,B_q\cup\{z\})-s(q,B_q)\right].
\label{eq:marginal-success}
\end{equation}
We also evaluate standalone success $U(\{z\})$, the success rate given only that entry, while $\mathrm{MSR}$ measures change relative to retrieved guidance. For source context $x_i$, candidates and all three conditions (standalone, bank-only, and bank-plus-candidate) share the reward-task set $\mathcal{D}=\mathcal{D}_i$, with $B_q$ fixed per task. The candidate is appended without replacing retrieved entries. An empty bank gives a no-knowledge baseline. Unlike ideal coverage, execution gains can be negative under interference. They operationalize incremental utility without estimating the coverage integral.

Between updates, eligible valid entries absent from the bank are ranked by positive $\mathrm{MSR}$ and admitted under a quota. If none has positive gain, the bank is unchanged. The updated bank is used in the next training batch; Algorithm~\ref{alg:knowledge-weaver} gives the schedule.

\subsection{MI-Inspired Information Feedback}
\label{sec:token-mi}
Marginal success feedback evaluates what an entry contributes beyond the existing bank. We complement this signal with information feedback that encourages the curator to retain the distinct information contained in its input experience. Mutual information, $I(X;Z)=H(Z)-H(Z\mid X)$, with $H$ denoting entropy, motivates measuring how source context $X$ changes uncertainty about generated knowledge $Z$ \citep{cover1991elements}. For an entry $z=(z_1,\ldots,z_T)$ generated from $x_i$ among $N\ge2$ contexts, we contrast token probabilities under source and alternative contexts,
\begin{equation}
\begin{gathered}
\delta_t = \log\frac{(N-1)\pi_{\mathrm{old}}(z_t\mid x_i,z_{<t})}{\sum_{\ell\ne i}\pi_{\mathrm{old}}(z_t\mid x_\ell,z_{<t})}, \quad r_t^{\mathrm{MI}} = b_t\min\left\lbrace c_{\mathrm{MI}},\lambda_{\mathrm{MI}}\lvert\delta_t\rvert\right\rbrace.
\end{gathered}
\label{eq:local-mi-credit}
\end{equation}
The token and its prefix are held fixed when comparing source likelihood with the mean likelihood under alternative contexts. We take the absolute value to reward the magnitude of input dependence, whether the source context increases or decreases a token's likelihood. Generic tokens with similar probabilities across contexts have small $|\delta_t|$ and receive little credit, whereas experience-specific tokens with larger contrasts receive more. The mask $b_t$ limits this credit to knowledge-body tokens, with scale $\lambda_{\mathrm{MI}}$ and cap $c_{\mathrm{MI}}$. Assigning credit locally favors input-specific content over generic wording, encouraging the curator to retain more information from each source experience.

% MSR 和 progressive的写法
\paragraph{Joint feedback and policy update.}
We construct the joint feedback by extending the SR-only task reward used in our GRPO baseline. SR rewards an entry's standalone utility without accounting for overlap with existing knowledge. We first add MI-inspired token credit to encourage preservation of information specific to each source experience. Because downstream evaluation concerns the utility of a knowledge bank, we further incorporate MSR to reward contributions beyond its existing entries. A formatting term additionally rewards valid structure and penalizes malformed outputs. Retaining GRPO's group-relative normalization (Equation~\ref{eq:grpo-advantage}), we combine these signals in a single token-level advantage,
\begin{equation}
A_{g,t}=\frac{\lambda_{\mathrm{SR}}U(\{z^{(g)}\})+\lambda_{\mathrm{MSR}}\mathrm{MSR}(z^{(g)})+r_{\mathrm{fmt}}(z^{(g)})-\bar R_{\mathrm{exec}}}{\sigma_{\mathrm{exec}}+\epsilon_{\mathrm{num}}}+r^{\mathrm{MI}}_{g,t},
\label{eq:combined-advantage}
\end{equation}
where $\bar R_{\mathrm{exec}}$ and $\sigma_{\mathrm{exec}}$ are the within-group mean and standard deviation of the combined SR, MSR, and formatting reward. The weights $\lambda_{\mathrm{SR}}$ and $\lambda_{\mathrm{MSR}}$ control the two utility terms. MSR augments GRPO's reward comparison with each entry's contribution to the bank, while MI supplies local credit for source-dependent content. Invalid candidates receive only the formatting penalty and no MI credit.  We retain GRPO for policy optimization, using the combined feedback to update only the curator through its clipped objective \citep{shao2024deepseekmath}. Optimization details are provided in Appendix~\ref{app:optimizer-details}.

%% file: sections/experiments.tex
\section{Experiments}
\label{sec:experiments}
\subsection{Experimental Setup}
\label{sec:experiment-setup}

\paragraph{Environments.}
We evaluate our method on two benchmarks: ALFWorld \citep{shridhar2020alfworld} and WebShop \citep{yao2022webshop}.
ALFWorld is a text-based environment in which an agent completes household tasks through sequential natural-language interactions. We evaluate knowledge banks on the \texttt{valid\_unseen} split across six canonical task types: Pick, Clean, Heat, Cool, Look, and Pick Two. We report the success rate (SR) for each task type and across all evaluation tasks.
WebShop is a simulated web shopping environment in which an agent searches for products, navigates product pages, and purchases an item that satisfies a user's requirements. We sample 100 tasks from the test set to evaluate the knowledge banks and report both the environment score and SR.

\paragraph{Training and Evaluation Protocol.}
Throughout training, we freeze the executor and optimize only the curator. We assign equal weights of $0.5$ to standalone SR and marginal SR. Token-level MI credit is scaled by $0.1$ and capped at $0.15$ after scaling. We use Qwen3.5-4B as the training model. After training, we use the trained curator to construct knowledge banks in a separate generation stage. We then evaluate the resulting banks with the frozen executor. 
Each knowledge bank contains 50 entries(except for SkillRL), with each entry limited to 100 words. To generate each entry, we randomly sample four trajectories from a trajectory pool.
For each evaluation task, we retrieve the top-$k$ entries using the \texttt{all-mpnet-base-v2} encoder \citep{reimers2019sentence}, ranking entries by their cosine similarity to the initial task query. Retrieval occurs once, before the first action, and the selected entries remain fixed throughout the episode.
We use $k=10$ for all banks in the main comparison and evaluate every $k\in\{0,\ldots,15\}$ in the scaling analysis.

\paragraph{Baselines.}
We compare four categories of knowledge banks, following Table~\ref{tab:main-results}:
(1) \emph{LLM Curators}: banks generated directly by Qwen3.5-4B, GPT-5 \citep{openai2025gpt5}, Claude Opus 4.8 \citep{anthropic2026claudeopus48}, and Gemini 3.7 Flash \citep{google2026gemini37flash},
(2) \emph{Established Knowledge Banks}: publicly released banks from SkillRL \citep{xia2026skillrl} and human-written banks refined by human,
(3) \emph{RL-Trained}: GRPO \citep{shao2024deepseekmath} baseline, and SkillRise \citep{yao2026skillrise} and Skill1 \citep{shi2026skill1}, which use execution feedback to improve knowledge generation and refinement,
(4) \emph{Prompt-Based}: ExpeL \citep{zhao2024expel}, which iteratively extracts and refines insights from successful and failed trajectories, and Mem0 \citep{chhikara2025mem0}, which extracts and organizes memories for later retrieval.
We also include zero-shot execution ($k$=0) as a reference without external knowledge.

\subsection{Main Results}
\label{sec:main-results}
% \paragraph{Knowledge Curation Benefits from Training beyond Task Success.}
\paragraph{Improved Performance across Benchmarks.}
Table~\ref{tab:main-results} shows that our method achieves the highest values on all three aggregate metrics with the same frozen executor at $k=10$. Overall SR reaches 54.0\% on ALFWorld and 42.0\% on WebShop, exceeding GRPO by 16.9 and 18.7 percentage points. On WebShop, its mean environment score reaches 66.4, compared with 49.0 for GRPO, indicating higher average reward alongside improved task success. Our method also outperforms SkillRise and Skill1 on both benchmarks. It exceeds SkillRL, the strongest ALFWorld baseline, by 2.8 percentage points in overall SR. Human-written banks achieve the highest baseline mean environment score on WebShop, but our method surpasses them on all three aggregate metrics. Gains over GRPO span five of six ALFWorld task types, although performance on Pick Two remains lower. These results support training knowledge curators with feedback beyond standalone task success.

% human-written 具体是怎么生成的，放在附录
\paragraph{Strong Base Models Do Not Guarantee Useful Knowledge.}
In Table~\ref{tab:main-results}, we evaluate strong general-purpose models as knowledge curators with the same frozen executor. Our trained Qwen3.5-4B curator outperforms GPT-5, Claude Opus 4.8, and Gemini 3.7 Flash on all three aggregate metrics. On ALFWorld, GPT-5 and Claude banks achieve overall SRs of 40.3\% and 37.1\%, respectively, above the zero-shot SR of 24.6\%. On WebShop, however, their banks yield SRs of 18.7\% and 24.7\%, below the zero-shot SR of 27.3\%. Gemini achieves the highest WebShop SR among these general-purpose curators (35.7\%), whereas its overall ALFWorld SR matches the zero-shot mean (24.6\%). These results show that strong general capabilities do not necessarily translate into knowledge that benefits downstream execution.
\input{tables/main_results}

\subsection{Knowledge Bank Quality Analysis}
\label{sec:knowledge-quality}

\begin{figure}[htbp]
    \centering
    \includegraphics[width=\linewidth]{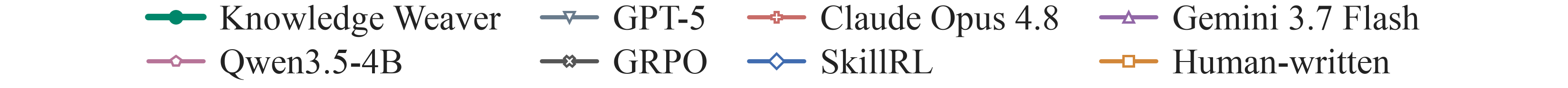}
    \par
    \begin{minipage}[t]{0.49\linewidth}
        \centering
        \includegraphics[width=\linewidth]{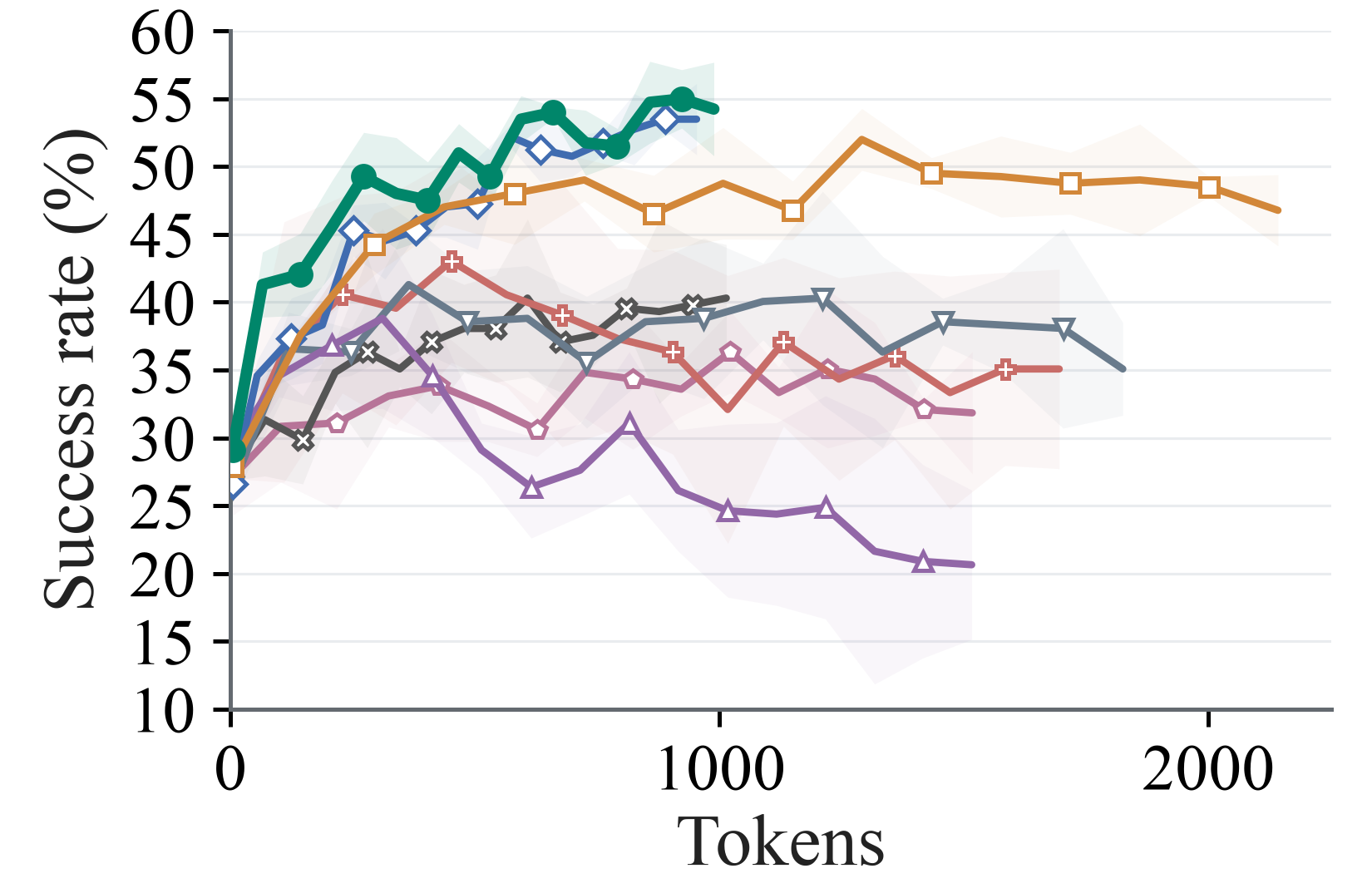}
        \par (a) ALFWorld
    \end{minipage}\hfill
    \begin{minipage}[t]{0.49\linewidth}
        \centering
        \includegraphics[width=\linewidth]{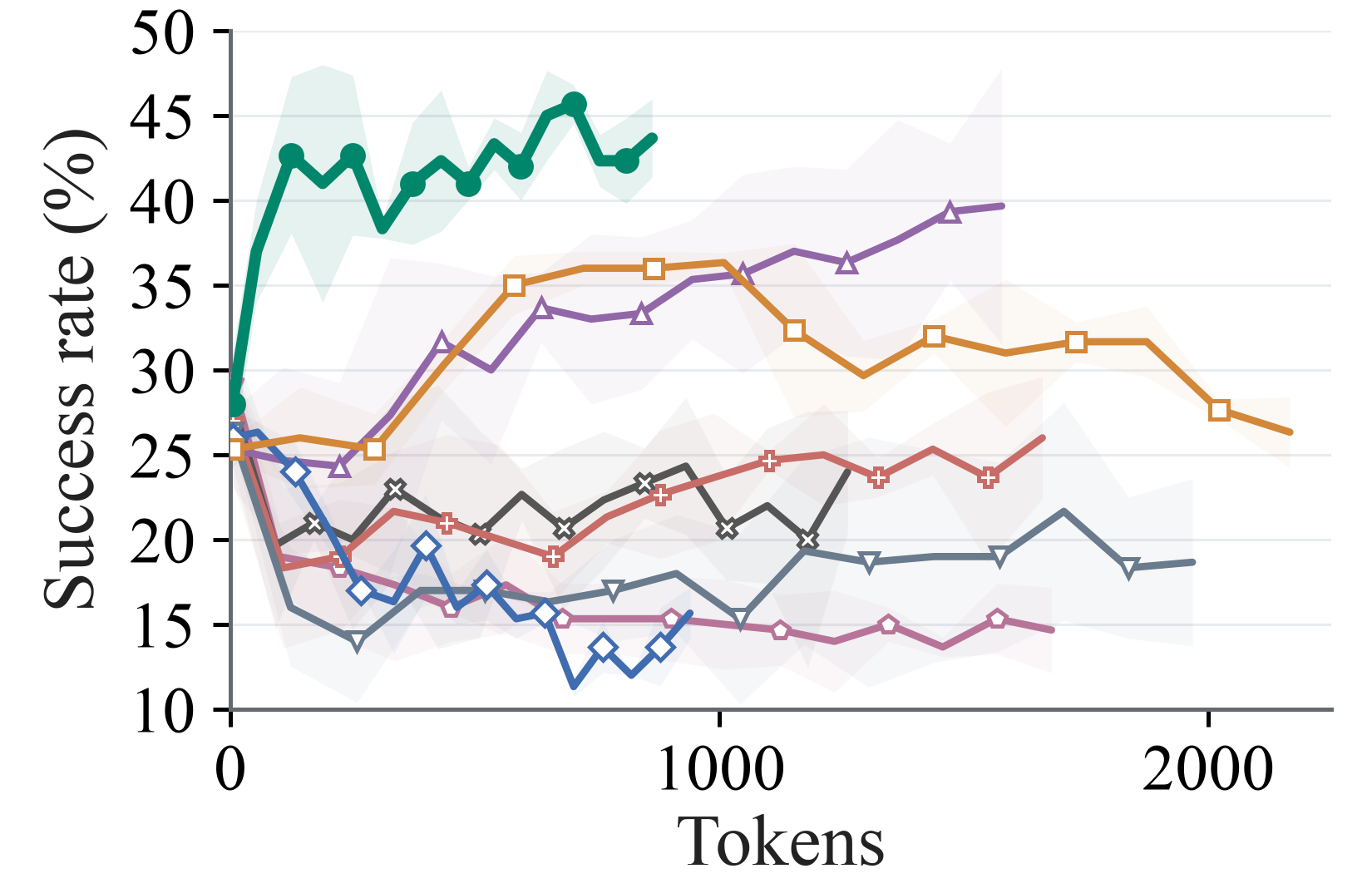}
        \par (b) WebShop
    \end{minipage}
    \vspace{-1mm}
    \caption{Success rate as the number of retrieved knowledge entries $k$ varies from 0 to 15. The horizontal axis reports the number of tokens in the formatted knowledge block injected into the frozen executor, averaged over tasks and runs. Lines and shading show the mean and one standard deviation across three runs. Each curve uses its own no-knowledge control at $k=0$.}
    \label{fig:knowledge-tokens}
\end{figure}

\paragraph{Useful Knowledge Requires Relatively Little Context.}
Figure~\ref{fig:knowledge-tokens} examines whether the main-table gains persist when the number of retrieved entries changes. Knowledge Weaver achieves the highest mean WebShop SR among the eight plotted methods at every positive retrieval count. At $k=10$, its knowledge context averages approximately 660 tokens on ALFWorld and 594 on WebShop, less than half the corresponding human-written context lengths. It nevertheless achieves higher SR on both benchmarks. The advantage therefore accompanies compact guidance, rather than a larger volume of injected text. As retrieval increases from $k=9$ to $15$, SR remains within 51.5-55.0\% on ALFWorld and 42.0-45.7\% on WebShop. The gains persist across this range despite fluctuations, reducing dependence on a single retrieval setting. ALFWorld performance is less uniformly dominant: human-written banks perform better at several smaller retrieval counts, and SkillRL slightly leads at $k=12$.

\paragraph{Additional Knowledge Can Help or Hinder Execution.}
Increasing retrieval exposes clear differences in how banks support the executor. On ALFWorld, Gemini's SR declines from 38.8\% at $k=3$ to 20.6\% at $k=15$, whereas Knowledge Weaver largely retains its gains. On WebShop, Qwen3.5-4B, GPT-5, Claude, and GRPO remain below their own no-knowledge controls at every positive retrieval count. Their additional context therefore fails to recover the performance lost when guidance is introduced. This pattern is consistent with unhelpful or conflicting guidance, although the curves alone cannot distinguish interference from retrieval mismatch or other content effects. Effective accumulation consequently requires assessing how entries work together, as well as whether individual entries appear useful. This observation motivates the marginal-utility objective and its contribution is examined alongside MI in the ablation. These sweeps retrieve from fixed banks, with observed token lengths rather than matched token budgets.   To complement these execution trends, we quantify whether different knowledge entries tend to be retrieved for the same tasks. At $k=10$, we compute mean pairwise Jaccard overlap between these task sets for entries retrieved at least once, then average over three banks per method. Overlap is lower for Knowledge Weaver than for GRPO, with values of 12.92\% versus 13.58\% on ALFWorld and 15.35\% versus 24.48\% on WebShop. Thus, sustained execution gains with longer context coexist with lower retrieval overlap at $k=10$. These observations motivate our joint use of MI feedback and marginal success rewards to preserve useful information, with execution benefits examined in Section~\ref{sec:ablations}.

\vspace{-3mm}
\subsection{Ablation Studies}
\label{sec:ablations}

\input{tables/reward_ablation}

\paragraph{MI and MSR Improve Knowledge Utility.}
Table~\ref{tab:reward-components} compares SR-only training with two progressively augmented objectives, SR + MI and SR + MI + MSR. The first extension adds token-level MI, increasing ALFWorld SR from 37.1\% to 46.0\% and WebShop SR from 23.3\% to 40.7\%. WebShop score also improves from 49.0 to 65.8, supporting information feedback as a useful addition to standalone-success training.

Building on SR + MI, the full objective adds MSR to reward contributions beyond existing knowledge. ALFWorld SR increases further to 54.0\%, an additional 8.0 percentage points, comparable to the 8.9-point improvement from adding MI. On WebShop, SR reaches 42.0\% and score reaches 66.4, with a smaller additional SR gain of 1.3 points. Both extensions therefore improve performance, with similarly sized SR gains on ALFWorld and a larger gain from MI on WebShop.

% 标题misleading
\vspace{-3mm}
\paragraph{Knowledge Banks Transfer across Frozen Executors.}
We reuse banks generated by a curator trained with feedback from a frozen Qwen3.5-4B executor. Table~\ref{tab:executor-transfer} reports their performance with eight additional frozen executors at $k=8$. Knowledge Weaver improves ALFWorld SR and both WebShop metrics over no-knowledge execution for all eight executors. Relative to GRPO, it achieves higher ALFWorld SR and WebShop scores for seven of eight executors on each metric, and higher WebShop SR for six, with one tie. GRPO banks improve ALFWorld SR for all eight executors but reduce WebShop SR for four. These results support reusing Knowledge Weaver banks across the evaluated executors without retraining the curator for each target model.
\input{tables/executor_transfer}
% \paragraph{Learned Knowledge Transfers beyond the \lz{Training Executor **?}.}
% The completed comparisons in Table~\ref{tab:executor-transfer} test whether knowledge curated with Qwen3.5-4B also benefits other frozen executors at $k=8$. \lz{Knowledge Weaver improves both benchmarks over no-knowledge execution for Qwen3.5-27B, DeepSeek-V4-Flash, MiniMax-M2.5, and MiniMax-M2.1. **?} Its WebShop banks also outperform GRPO in both score and SR for these executors. The benefits therefore extend beyond the model used to provide training feedback, without retraining the curator for each executor. On ALFWorld, however, GRPO slightly exceeds Knowledge Weaver for Qwen3.5-27B, despite the opposite ordering on WebShop. Transfer thus depends on the interaction between the bank, executor, and environment. These results support reuse across the evaluated models, while leaving the generality of the gains to the pending executor comparisons.

%% file: tables/main_results.tex
\begin{table}[t]
\centering
\small
\setlength{\tabcolsep}{2.5pt}
\renewcommand{\arraystretch}{1}
\caption{Performance of knowledge banks with a frozen Qwen3.5-4B executor on ALFWorld and WebShop. Each bank supplies $k=10$ retrieved entries per task and zero-shot execution uses no external knowledge. We report success rate (SR, \%) and mean environment score on WebShop. All results are averaged over 3 seeds, with standard deviations shown as subscripts for the Overall and WebShop columns. \textbf{Bold} and \secondbest{underlining} mark the best and second-best in each column.}
\label{tab:main-results}
\vspace{2pt}
\begin{tabularx}{0.9\textwidth}{@{}>{\raggedright\arraybackslash}X*{9}{r}@{}}
\toprule
& \multicolumn{7}{c}{\textbf{ALFWorld} (SR\%)} & \multicolumn{2}{c}{\textbf{WebShop}} \\
\cmidrule(r){2-8}\cmidrule(l){9-10}
\textbf{Method} & \multicolumn{1}{c}{Pick} & \multicolumn{1}{c}{Clean} & \multicolumn{1}{c}{Heat} & \multicolumn{1}{c}{Cool} & \multicolumn{1}{c}{Look} & \multicolumn{1}{c}{Pick Two} & \multicolumn{1}{c}{Overall} & \multicolumn{1}{c}{Score} & \multicolumn{1}{c@{}}{SR\%} \\
\midrule
Zero-shot & $30.6$ & $10.8$ & $31.9$ & $30.2$ & $24.1$ & $25.5$ & $24.6_{\pm 1.5}$ & $53.5_{\pm 1.9}$ & $27.3_{\pm 1.5}$ \\
\midrule
\addlinespace[3pt]
\rowcolor{gray!15}[0pt][0pt]
\multicolumn{10}{@{}l}{\scshape LLM curators} \\
Qwen3.5-4B & $40.3$ & $43.0$ & $24.6$ & $41.3$ & $48.1$ & $15.7$ & $36.3_{\pm 3.0}$ & $40.6_{\pm 2.2}$ & $14.7_{\pm 2.1}$ \\
GPT-5 & $40.3$ & $39.8$ & $27.5$ & $\underline{61.9}$ & $50.0$ & $21.6$ & $40.3_{\pm 7.9}$ & $45.2_{\pm 6.8}$ & $18.7_{\pm 7.4}$ \\
Claude Opus 4.8 & $40.3$ & $30.1$ & $40.6$ & $36.5$ & $\underline{72.2}$ & $3.9$ & $37.1_{\pm 6.2}$ & $55.2_{\pm 2.5}$ & $24.7_{\pm 0.6}$ \\
Gemini 3.7 Flash & $25.0$ & $11.8$ & $18.8$ & $25.4$ & $64.8$ & $11.8$ & $24.6_{\pm 6.4}$ & $61.3_{\pm 5.4}$ & $35.7_{\pm 5.9}$ \\
\addlinespace[3pt]
\rowcolor{gray!15}[0pt][0pt]
\multicolumn{10}{@{}l}{\scshape Established Knowledge Banks} \\
SkillRL & $34.7$ & $\mathbf{50.5}$ & $50.7$ & $\mathbf{68.3}$ & $66.7$ & $\mathbf{39.2}$ & $51.2_{\pm 2.4}$ & $36.6_{\pm 3.0}$ & $15.7_{\pm 0.6}$ \\
Human-written & $\mathbf{47.2}$ & $\secondbest{48.4}$ & $\mathbf{63.8}$ & $52.4$ & $55.6$ & $25.5$ & $49.5_{\pm 1.1}$ & $63.1_{\pm 1.5}$ & $32.0_{\pm 1.0}$ \\
\addlinespace[3pt]
\rowcolor{gray!15}[0pt][0pt]
\multicolumn{10}{@{}l}{\scshape RL-Trained} \\
GRPO & $30.6$ & $26.9$ & $34.8$ & $44.4$ & $55.6$ & $\mathbf{39.2}$ & $37.1_{\pm 3.7}$ & $49.0_{\pm 4.8}$ & $23.3_{\pm 2.1}$ \\
% SkillRise & $44.4$ & $21.5$ & $33.3$ & $27.0$ & $38.9$ & $31.4$ & $32.1_{\pm 2.2}$ & $40.0_{\pm 5.6}$ & $18.7_{\pm 3.5}$ \\
SkillRise & $43.8$ & $22.6$ & $32.6$ & $28.6$ & $47.2$ & $29.4$ & $33.2_{\pm 1.6}$ & $40.0_{\pm 5.6}$ & $18.7_{\pm 3.5}$ \\
Skill1 & $37.5$ & $25.8$ & $14.5$ & $34.9$ & $44.4$ & $25.5$ & $29.9_{\pm 8.2}$ & $38.5_{\pm 0.9}$ & $17.0_{\pm 1.0}$ \\
\addlinespace[3pt]
\rowcolor{gray!15}[0pt][0pt]
\multicolumn{10}{@{}l}{\scshape Prompt-based} \\
ExpeL & $38.9$ & $18.3$ & $23.2$ & $27.0$ & $37.0$ & $\secondbest{31.4}$ & $28.4_{\pm 2.7}$ & $58.4_{\pm 1.2}$ & $24.7_{\pm 2.1}$ \\
Mem0 & $36.8$ & $37.1$ & $\underline{52.2}$ & $47.6$ & $47.2$ & $25.5$ & $41.2_{\pm 0.9}$ & $50.1_{\pm 1.3}$ & $19.0_{\pm 2.0}$ \\
\midrule
\rowcolor{blue!8}
\textbf{Our} & $\secondbest{45.8}$ & $\secondbest{48.4}$ & $\mathbf{63.8}$ & $54.0$ & $\mathbf{87.0}$ & $27.5$ & $\mathbf{54.0}_{\pm 0.4}$ & $\mathbf{66.4}_{\pm 2.6}$ & $\mathbf{42.0}_{\pm 2.0}$ \\
\bottomrule
\end{tabularx}
\vspace{-5mm}
\end{table}

%% file: tables/reward_ablation.tex
\begin{wraptable}{R}{0.48\textwidth}
\centering
\small
\caption{Reward ablation at $k=10$ with a frozen Qwen3.5-4B executor. SR rewards individual task success, MSR rewards success-rate gains over retrieved knowledge, and MI provides token-level information feedback. Formatting rewards are shared across configurations. Values are three-run means under different seeds. \textbf{Bold} marks column best.}
\vspace{+0.1cm}
\label{tab:reward-components}
\begin{tabular}{@{}lccc@{}}
\toprule
Rewards & ALFWorld & \multicolumn{2}{c}{WebShop} \\
\cmidrule(lr){2-2}\cmidrule(l){3-4}
& \multicolumn{1}{c}{SR} & \multicolumn{1}{c}{Score} & \multicolumn{1}{c@{}}{SR} \\
\midrule
SR & 37.1 & 49.0 & 23.3 \\
SR + MI & 46.0 & 65.8 & 40.7 \\
SR + MI + MSR & \textbf{54.0} & \textbf{66.4} & \textbf{42.0} \\
\bottomrule
\end{tabular}
\end{wraptable}

%% file: tables/executor_transfer.tex
\begin{table}[htbp]
\centering
\small
\setlength{\tabcolsep}{2.5pt}
\caption{Transfer of knowledge banks across frozen executors. None denotes execution without knowledge ($k=0$), while GRPO and Knowledge Weaver (Ours) retrieve $k=8$ entries per task. Knowledge-augmented results are averaged over three banks per method. \textbf{Bold} marks the best result for each executor and metric.}
%\lz{**less important results, maybe can adjust the space used? also, all use K=10 before, why use 8 here? seems k = 0 results are expected? Abnormal results for GRPO should be discussed}}
\label{tab:executor-transfer}
\vspace{2mm}
\begin{tabular*}{0.9\linewidth}{@{\extracolsep{\fill}}l*{9}{r}@{}}
\toprule
\textbf{Executor}
& \multicolumn{3}{c}{\textbf{ALFWorld SR (\%)}}
& \multicolumn{3}{c}{\textbf{WebShop Score}}
& \multicolumn{3}{c}{\textbf{WebShop SR (\%)}} \\
\cmidrule(lr){2-4}\cmidrule(lr){5-7}\cmidrule(l){8-10}
& None & GRPO & Ours & None & GRPO & Ours & None & GRPO & Ours \\
\midrule
Qwen3.5-4B & 24.6 & 38.1 & \textbf{49.3} & 53.5 & 49.0 & \textbf{67.5} & 27.3 & 20.7 & \textbf{41.0} \\
\midrule
Qwen3.5-27B & 48.5 & \textbf{61.9} & 61.2 & 54.9 & 55.6 & \textbf{66.9} & 25.0 & 23.0 & \textbf{41.7} \\
Llama-3-8B-Instruct & 6.7 & 7.2 & \textbf{12.7} & 2.7 & 10.3 & \textbf{10.5} & 0.0 & \textbf{4.7} & 4.0 \\
Llama-3.1-8B-Instruct & 9.7 & 14.9 & \textbf{21.1} & 16.0 & 6.7 & \textbf{19.7} & 6.0 & 2.3 & \textbf{11.7} \\
Phi-4-mini-instruct & 7.5 & 9.5 & \textbf{13.9} & 2.3 & \textbf{8.9} & 5.4 & 1.0 & \textbf{3.3} & \textbf{3.3} \\
Phi-4 & 19.4 & 32.6 & \textbf{38.1} & 57.0 & 51.3 & \textbf{65.2} & 13.0 & 16.7 & \textbf{33.3} \\
DeepSeek-V4-Flash & 56.0 & 64.2 & \textbf{64.9} & 57.0 & 56.2 & \textbf{69.2} & 30.0 & 26.0 & \textbf{46.3} \\
MiniMax-M2.5& 35.1 & 53.7 & \textbf{56.5} & 62.8 & 64.8 & \textbf{73.6} & 30.0 & 33.0 & \textbf{48.0} \\
MiniMax-M2.1& 70.9 & 73.9 & \textbf{74.4} & 73.2 & 68.7 & \textbf{75.8} & 43.0 & 37.0 & \textbf{49.7} \\
\bottomrule
\end{tabular*}
\vspace{-3mm}
\end{table}

%% file: sections/conclusion.tex
\vspace{-3mm}
\section{Conclusion}
We introduced Knowledge Weaver, an RL framework that trains knowledge curators to preserve information through token-level MI feedback and rewards for marginal gains over retrieved knowledge, while rewarding standalone success. On ALFWorld and WebShop, the resulting banks outperform the evaluated baselines with a frozen executor, and the gains hold across retrieval sizes. Ablation indicates that MI adds value beyond task success alone. The banks also transfer to other executors, though the gains vary by model. These results support information preservation as a useful principle for building reusable knowledge from agent experience.

We plan several directions that follow from this work in the future. First, because our experiments keep the executor frozen, a natural extension is to jointly train the curator and executor, allowing knowledge construction and utilization to improve in tandem. Second, the uneven transfer gains suggest that knowledge effective for one executor may not transfer directly to another, motivating the development of knowledge banks that generalize across models. Finally, extending Knowledge Weaver to longer-horizon tasks, where the bank is continually updated as experience accumulates, would provide a stronger test of how well knowledge preservation scales over time. More broadly, we view learning to curate reusable knowledge as a step toward agents that self-improve continually through their own experience.

%% file: sections/appendix/proofs.tex
\section{Proofs}
\label{app:proofs}

We use the Gaussian utility profiles from Equation~\ref{eq:utility-profile}, with $K\ge2$ fixed centers $c_1<\cdots<c_K$ and common height $h>0$ and scale $\sigma>0$. Specializing Equation~\ref{eq:knowledge-quality} gives the envelope coverage $\mathcal{E}(h,\sigma)=\int_{\mathbb{R}}\max_i f_i(u;h,\sigma)\,du$. Write $\Delta_i=c_{i+1}-c_i>0$ for the adjacent gaps, and let $\phi$ and $\Phi$ denote the standard normal density and cumulative distribution function.

\subsection{Proof of Proposition~\ref{prop:height}}
\label{app:proof-prop1}

\begin{proof}
For fixed $\sigma$, changing the common height rescales every profile by the same positive factor. This factor can be taken outside the pointwise maximum and the integral, giving
\begin{equation}
\label{eq:homogeneity}
\mathcal{E}(h,\sigma) = h\int_{\mathbb{R}}\max_{1\le i\le K} \exp\!\left(-\frac{(u-c_i)^2}{2\sigma^2}\right)\,du = h\,\mathcal{E}(1,\sigma).
\end{equation}
To establish strict monotonicity, it remains to check that the coefficient $\mathcal{E}(1,\sigma)$ is positive and finite. The envelope is bounded below by any individual profile and above by their sum. Since each unit-height profile integrates to $\sqrt{2\pi}\,\sigma$, these bounds imply
\[
0<\sqrt{2\pi}\,\sigma \le\mathcal{E}(1,\sigma) \le K\sqrt{2\pi}\,\sigma<\infty.
\]
Thus coverage is linear and strictly increasing in the common height $h$.
\end{proof}

\subsection{Proof of Proposition~\ref{prop:width}}
\label{app:proof-prop2}

\begin{proof}
Fix the per-entry mass $A>0$. The Gaussian integral gives $A=\sqrt{2\pi}\,\sigma h$, so the height must vary with $\sigma$ according to
\begin{equation}
\label{eq:hA}
h_A(\sigma)=\frac{A}{\sqrt{2\pi}\,\sigma}, \qquad f_i(u;h_A(\sigma),\sigma) =\frac{A}{\sigma}\phi\!\left(\frac{u-c_i}{\sigma}\right).
\end{equation}
Each profile therefore has mass $A$ for every $\sigma>0$. We write $\mathcal{E}_A(\sigma)=\mathcal{E}(h_A(\sigma),\sigma)$ for the coverage under this constraint.

Because all profiles share the same height and scale, the nearest center provides the largest utility at each location. Define the adjacent midpoints $m_i=(c_i+c_{i+1})/2$ for $i=1,\ldots,K-1$, with $m_0=-\infty$ and $m_K=+\infty$. On $[m_{i-1},m_i]$, the center $c_i$ is nearest to $u$, so $f_i$ attains the upper envelope. We can therefore integrate the envelope one interval at a time:
\begin{align}
\label{eq:decomposition}
\mathcal{E}_A(\sigma) &=\sum_{i=1}^{K}\int_{m_{i-1}}^{m_i} f_i(u;h_A(\sigma),\sigma)\,du \notag\\
&=A\sum_{i=1}^{K}\left[ \Phi\!\left(\frac{m_i-c_i}{\sigma}\right) -\Phi\!\left(\frac{m_{i-1}-c_i}{\sigma}\right) \right].
\end{align}
For each interior interval, the upper and lower bounds relative to its center are $\Delta_i/2$ and $-\Delta_{i-1}/2$. Using $\Phi(-z)=1-\Phi(z)$, its contribution is $A[\Phi(\Delta_i/(2\sigma))+\Phi(\Delta_{i-1}/(2\sigma))-1]$. The two end intervals contribute $A\Phi(\Delta_1/(2\sigma))$ and $A\Phi(\Delta_{K-1}/(2\sigma))$, respectively. Thus each adjacent gap contributes twice, while each of the $K-2$ interior intervals contributes one subtraction of $A$. Summing gives
\begin{equation}
\label{eq:closed-form}
\mathcal{E}_A(\sigma) =A\left[2\sum_{i=1}^{K-1} \Phi\!\left(\frac{\Delta_i}{2\sigma}\right)-(K-2)\right].
\end{equation}
This expression also holds for $K=2$, when there are no interior intervals. Differentiating with $A$ and the centers fixed, and using $\Phi'=\phi$, yields
\begin{equation}
\label{eq:derivative}
\frac{d\mathcal{E}_A(\sigma)}{d\sigma} =-\frac{A}{\sigma^2}\sum_{i=1}^{K-1}\Delta_i\, \phi\!\left(\frac{\Delta_i}{2\sigma}\right)<0.
\end{equation}
The sum contains at least one term because $K\ge2$, and every term is positive because $\Delta_i>0$ and $\phi$ is strictly positive. Hence coverage strictly decreases with $\sigma$ under the fixed per-entry mass constraint, proving the proposition.
\end{proof}

\paragraph{Redundancy and marginal coverage.}
At each location, utility counted by the sum of profiles but omitted by their maximum represents redundant overlap. Its integral is $\mathcal{R}_A(\sigma)=\int_{\mathbb{R}}[\sum_i f_i(u;h_A(\sigma),\sigma)-\max_i f_i(u;h_A(\sigma),\sigma)]\,du$. Subtracting Equation~\ref{eq:closed-form} from the total mass $KA$ gives
\begin{equation}
\label{eq:redundancy}
\mathcal{R}_A(\sigma) =KA-\mathcal{E}_A(\sigma) =2A\sum_{i=1}^{K-1}\Phi\!\left(-\frac{\Delta_i}{2\sigma}\right).
\end{equation}
Since $\mathcal{R}_A'=-\mathcal{E}_A'>0$, increasing the common scale increases redundant overlap. As $\sigma\to0^+$, each $\Phi$ term in Equation~\ref{eq:closed-form} tends to $1$, giving coverage $KA$ in the limit. As $\sigma\to\infty$, each term tends to $1/2$, giving limiting coverage $A$. Consequently, the coverage efficiency $\mathcal{E}_A/(KA)$ decreases from a limit of $1$ to a limit of $1/K$.

More generally, for any nonnegative, integrable candidate profile $f_z$, its marginal coverage is
\begin{equation}
\begin{aligned}
\Delta_{\mathrm{cov}}(z\mid\mathcal{C})
&=\mathcal{E}(\mathcal{C}\cup\{z\})-\mathcal{E}(\mathcal{C})\\
&=\int_{\mathbb{R}}\left[f_z(u)-\max_{1\le i\le K}f_i(u)\right]_+\,du,
\end{aligned}
\label{eq:marginal-coverage}
\end{equation}
where $[a]_+=\max(a,0)$. This follows pointwise from $\max\{a,b\}-b=[a-b]_+$. Apply the identity with $a=f_z(u)$ and $b=\max_i f_i(u)$, then integrate. It holds for any nonnegative, integrable profiles and requires no Gaussian assumption.

For the Gaussian specialization, consider adding an entry with the same $A$ and $\sigma$ at $c_K+\Delta$, where $\Delta>0$. The coverage formula gains one gap term and one additional subtraction of $A$, so the marginal coverage is $A[2\Phi(\Delta/(2\sigma))-1]$. It is strictly increasing in $\Delta$ and lies in $(0,A)$, approaching $0$ as $\Delta\to0^+$ and $A$ as $\Delta\to\infty$. The same argument applies to an entry added beyond the leftmost center.

%% file: sections/appendix/case_study.tex
\clearpage
\section{Case Studies of Knowledge Bank Content}
\label{app:knowledge-cases}

We inspect knowledge banks generated from step-40 curator checkpoints to examine what repeated and distinct entries contain. All examples below are complete entries from the GRPO and Knowledge Weaver used in Table~\ref{tab:main-results}. Each source bank contains 50 entries. We select cases from the knowledge banks of GRPO and Knowledge Weaver. 

\begingroup
\definecolor{casegrpo}{HTML}{8B5353}
\definecolor{caseours}{HTML}{167565}
\newtcolorbox{knowledgecase}[2]{%
  enhanced,breakable,
  colback=#1!3!white,colframe=#1!70!white,
  colbacktitle=#1!10!white,coltitle=#1!60!black,
  fonttitle=\normalsize\bfseries,fontupper=\normalsize,
  title={#2},boxrule=0.5pt,arc=1.5mm,
  left=2.5mm,right=2.5mm,top=1.5mm,bottom=1.5mm,
  toptitle=1.5mm,bottomtitle=1.5mm,
  before skip=7pt,after skip=7pt,
  before upper={\setlength{\parindent}{0pt}}}
\newcommand{\knowledgeentry}[3]{%
  \textbf{#2}\hfill{\normalsize\texttt{#1}}\par\nobreak
  #3\par\smallskip}

\subsection{ALFWorld}
\label{app:case-alfworld}

\begin{knowledgecase}{casegrpo}{GRPO \enspace\textbar\enspace Step 40}
\knowledgeentry{}{Possession before placement}{%
locate your target object, verify it is on the surface with take, then immediately navigate to the destination and execute move to place it. Avoid navigating without picking up the object or performing unnecessary inventory checks that delay the action.}

\knowledgeentry{}{Possession before placement}{%
always take the object from the source before moving or attempting to place it there, even if the action seems redundant, to ensure possession is maintained and progress is made.}

\knowledgeentry{}{Possession before placement}{%
Check your inventory before moving to the destination to confirm you are carrying the correct object, then proceed directly to the destination and place the object there immediately without returning to the source.}

\end{knowledgecase}

\begin{knowledgecase}{caseours}{Knowledge Weaver \enspace\textbar\enspace Step 40}
\knowledgeentry{}{Cooling with the fridge}{%
When cooling an object for placement, immediately pick it up from the countertop after navigating there, then head to the fridge and cool the object there, rather than returning to the countertop or searching for the object multiple times.}

\knowledgeentry{}{Heating with the microwave}{%
If you need to heat an item before placing it: verify the item is accessible by picking it up first (e.g., from the countertop), then heat it in the microwave using valid actions like move and heat; this ensures the object is in your possession and eliminates the error of heating an object that is still sitting in one location while you navigate to gain access to another, which consumes steps without progressing toward success.}

\knowledgeentry{}{Inspection under a lamp}{%
Before looking at an object in light, verify you are holding the object; if not, locate it by navigating to its current position and pick it up, then turn on the nearby lamp by using it only after the object is in your inventory.}

\end{knowledgecase}

The three GRPO entries reduce to the same possession-before-placement routine despite their different wording. The selected Knowledge Weaver entries retain distinct operations: cooling with a fridge, heating with a microwave, and activating a lamp while holding the object. These differences concern the prescribed action and appliance, rather than only an object-name substitution. They demonstrate task-specific content in individual entries, although bank-wide redundancy is not directly assessed by these examples.

\clearpage
\subsection{WebShop}
\label{app:case-webshop}

\begin{knowledgecase}{casegrpo}{GRPO \enspace\textbar\enspace Step 40}
\knowledgeentry{}{Verify color before purchase}{%
Before clicking Buy Now, explicitly verify the visible "color" filter has been applied by selecting a specific color option, rather than relying on assumptions that a search query string initializes this setting.}

\knowledgeentry{}{Verify color before purchase}{%
Before clicking the Buy Now button, explicitly verify that the selected color option has actually been applied by checking the instruction text for a confirmation message (e.g., "You have clicked dark gray") rather than relying on the absence of an error message. This prevents the "invisible update" failure where missing a confirmation causes the price filter to reset, allowing unselected colors to bypass the intended price constraint and wasting search steps.}

\knowledgeentry{}{Verify color before purchase}{%
Before clicking Buy Now, verify the selected color filter option is actually selected on the product interface-based on the URL or page state, not solely on the search query or user intent.}

\end{knowledgecase}

\begin{knowledgecase}{caseours}{Knowledge Weaver \enspace\textbar\enspace Step 40}
\knowledgeentry{}{Select a product attribute}{%
When a product page lists multiple color or size options alongside the current selection, always click the specific correct option to apply the attribute before clicking Buy Now.}

\knowledgeentry{}{Constrain the search query}{%
When combined product attributes like color or material type are critical to the search intent, explicitly specify these color and material constraints in the search query rather than relying solely on the displayed product page for confirmation, as clicking a product page to select a specific color attribute often triggers an error stating "INVALID(click target is not currently visible)."}

\knowledgeentry{}{Continue to the next page}{%
When a product page displays no purchasable option of the specified color or size, explicitly advance to the next page to find the correct attribute before clicking Buy Now, rather than leaving the system with a non-compliant selection field.}

\end{knowledgecase}

The GRPO examples provide general guidance on checking color before purchase, varying the suggested confirmation signal while repeating the same decision rule. Across this bank, 48 of 50 entries begin with ``Before''. The selected Knowledge Weaver entries cover three different decisions: applying an available attribute, expressing constraints in the search query, and continuing to another page when a suitable option is absent.
\endgroup

%% file: sections/appendix/prompt.tex
% Appendix section: replace the existing section with label app:agent-prompts.
% Required in the document preamble:
% \usepackage{xcolor}
% \usepackage{tcolorbox}
% \tcbuselibrary{skins,breakable}
% Prompt source revision: 2c3eb6a03195c5db39086607fbd14e85cf873658.
% Source locations and extracted strings: data/appendix_prompt_sources.json.
% All prompt boxes are included below; no external prompt file is required.
\section{Curator and Executor Prompts}
\label{app:agent-prompts}

We report the system and user message templates used for knowledge curation and execution on ALFWorld and WebShop. The curator receives four trajectories from the same task type and returns a single knowledge entry of at most 100 words, enclosed in \texttt{\textless{}skill\textgreater{}} and \texttt{\textless{}/skill\textgreater{}} tags. We reproduce the evidence-conditioned prompts, including their output requirements, and retain the implementation's term ``skill.'' The prompt wording is preserved; whitespace is adjusted for typesetting.

Fields such as \texttt{\{trajectories\}}, \texttt{\{observation\}}, and \texttt{\{guidance\}} represent content inserted at runtime. The trajectory field contains the rendered source histories, and the guidance field contains the knowledge supplied to the executor. For fixed-bank evaluation, multiple retrieved entries are numbered under ``Retrieved guidance skills:'' and remain fixed throughout the episode. The memory block is omitted when no prior interactions are available, and the WebShop invalid-response line is included only when retrying an invalid response. In the ALFWorld action vocabulary, \texttt{\{obj\}}, \texttt{\{recep\}}, and \texttt{\{appliance\}} are literal action-template slots.

The executor fields are shown after preprocessing and before model-specific chat formatting. The recorded evaluations retain up to five recent interactions. ALFWorld supplies the final 1,800 observation characters and at most 140 admissible commands; WebShop limits the current observation to 5,000 characters and lists at most 60 visible click targets. Command and click-target lists are sorted before truncation.

\begingroup
\raggedbottom
\definecolor{promptteal}{HTML}{167565}
\definecolor{promptblue}{HTML}{315E91}
\definecolor{promptorange}{HTML}{A55D22}
\definecolor{promptpurple}{HTML}{755394}
\newtcolorbox{promptbox}[2]{enhanced,breakable,ignore nobreak,colback=#1!3!white,colframe=#1!70!white,colbacktitle=#1!10!white,coltitle=#1!65!black,fonttitle=\normalsize\bfseries,fontupper=\normalsize,title={#2},title after break={#2 (continued)},boxrule=0.5pt,arc=1.5mm,left=2.5mm,right=2.5mm,top=1.5mm,bottom=1.5mm,toptitle=1.5mm,bottomtitle=1.5mm,before skip=7pt,after skip=7pt,before upper={\setlength{\parindent}{0pt}\raggedright}}
\subsection{ALFWorld curator}
\begin{promptbox}{promptteal}{ALFWorld curator | System}
You extract one useful guidance skill from agent trajectories. The acting agent already receives general operating guidance to consult its legal actions, navigate where needed, reveal accessible contents, and carry out the requested interaction. Treat that guidance as given rather than rewriting it as the learned skill.\par\medskip Study the complete trajectories as behavioral evidence. Keep investigating beyond convenient surface explanations until the histories support an additional consequential insight that can improve future decisions. The conclusion must depend on what actually happened in the trajectories. Do not merely summarize the task or action-space description, imitate a fixed writing pattern, or invent unsupported details. Reason privately and output only the final natural-language skill required by the output contract. Return exactly one skill wrapped by \textless{}skill\textgreater{} and \textless{}/skill\textgreater{}, with no text outside the tags.\par\medskip
\end{promptbox}

\begin{promptbox}{promptteal}{ALFWorld curator | User}
\# Task category\par \{category\}   (\{category\_note\})\par\medskip \# Legal action vocabulary\par The agent may use these action templates: go to \{recep\}, open/close \{recep\}, take \{obj\} from \{recep\}, put \{obj\} in/on \{recep\}, use \{obj\}, heat/cool/clean \{obj\} with \{appliance\}, examine \{obj\}, look. This list only constrains which actions are legal. It is not trajectory evidence and must not be restated as the learned skill. The agent only sees the current room view. Receptacles must be navigated to, and some must be opened before their contents become visible.\par\medskip \# Complete trajectories\par \{trajectories\}\par\medskip \# Skill content requirements\par - Read across the complete action histories and discover a consequential issue or insight that is easy to miss from the task description alone, then give useful guidance for addressing it.\par - Ground the guidance in differences, consequences, or lack of progress visible in the supplied trajectories. If the same advice could be written without reading them, investigate further.\par - Do not mechanically recap the action sequence.\par - Use only legal actions from the action vocabulary. Do not invent actions or unsupported facts.\par - Write one compact natural-language paragraph of at most 100 words, without following a fixed rhetorical template.\par\medskip \# Strict output requirements\par - Return exactly: \textless{}skill\textgreater{}\textless{}one natural-language guidance paragraph\textgreater{}\textless{}/skill\textgreater{}\par - The opening tag must be the first non-whitespace text and the closing tag the last.\par - Use exactly one opening tag and one closing tag; do not add analysis, markdown, labels, nested tags, or text outside the wrapper.\par\medskip
\end{promptbox}

\subsection{ALFWorld executor}
\begin{promptbox}{promptblue}{ALFWorld executor | System}
You are an ALFWorld TextWorld agent. Choose the next action using this exact format:\par Thought: \textless{}concise reasoning, max 30 words\textgreater{}\par Action: \textless{}copy exactly one command from the admissible command list\textgreater{}\par Write the Action line immediately after the Thought line. Do not add any other text.\par Guidance skill:\par \{guidance\}\par\medskip
\end{promptbox}

\begin{promptbox}{promptblue}{ALFWorld executor | User}
Task: \{task\_description\}\par Observation: \{observation\}\par Recent online episodic memory:\par \{online\_memory\}\par Admissible commands:\par \{admissible\_commands\}\par\medskip Respond with Thought and Action:\par\medskip
\end{promptbox}

\subsection{WebShop curator}
\begin{promptbox}{promptorange}{WebShop curator | System}
You extract one useful guidance skill from shopping-agent trajectories. Read the complete histories as behavioral evidence and begin with the concrete details that explain why an action helped, failed, or changed the outcome. Infer one non-obvious, consequential insight that would materially improve a future agent's decision in a similar situation. Let the evidence determine what the skill is about instead of starting from a predetermined workflow or restating advice that could be written without these histories.\par\medskip Make the insight specific enough to guide a decision. Ground every claim in an observed state, action consequence, or outcome difference. Do not produce generic boilerplate, a reusable checklist, a trajectory recap, or unsupported product facts. Reason privately and output only the final skill required by the output contract.\par\medskip Return exactly one \textless{}skill\textgreater{}...\textless{}/skill\textgreater{} wrapper and no text outside it.\par\medskip
\end{promptbox}

\begin{promptbox}{promptorange}{WebShop curator | User}
\# Task category\par WebShop product search, option selection, and purchase\par\medskip \# Legal action vocabulary\par The agent can call search[concise query] when a search box is present and click[exact visible target] for products, navigation controls, product options, and Buy Now. This vocabulary only constrains legal behavior; it is not trajectory evidence and must not be restated as the skill.\par\medskip \# Complete trajectory evidence\par \{trajectories\}\par\medskip \# Skill content requirements\par - Extract a consequential insight whose usefulness depends on reading these histories closely, including evidence from failed progress, misleading page state, or a successful change in behavior.\par - Prefer the smallest behaviorally meaningful correction or decision principle supported by the evidence. If the same advice could be written without the histories, investigate further.\par - Let concrete trajectory evidence determine both the content and wording. Do not fall back on generic boilerplate, a reusable checklist, or a trajectory recap.\par - Use only legal WebShop actions and facts supported by the supplied histories.\par - Write one compact natural-language paragraph of at most 100 words.\par\medskip \# Strict output requirements\par - Return exactly: \textless{}skill\textgreater{}\textless{}one natural-language guidance paragraph\textgreater{}\textless{}/skill\textgreater{}\par - The opening tag must be the first non-whitespace text and the closing tag the last.\par - Use exactly one opening tag and one closing tag. Do not add analysis, markdown, labels, nested tags, or text outside the wrapper.\par\medskip
\end{promptbox}

\subsection{WebShop executor}
\begin{promptbox}{promptpurple}{WebShop executor | System}
You are a WebShop text agent. Select the next action to satisfy the shopping instruction and maximize the final WebShop reward.\par Respond using exactly two lines:\par Thought: \textless{}concise reasoning, at most 35 words\textgreater{}\par Action: \textless{}one valid action\textgreater{}\par Use search[concise keywords] only when search is available. Use click[exact target] by copying one currently visible clickable target. Do not add any other text.\par Guidance skill:\par \{guidance\}\par\medskip
\end{promptbox}

\begin{promptbox}{promptpurple}{WebShop executor | User}
Shopping instruction: \{shopping\_instruction\}\par Current page observation:\par \{observation\}\par Recent online memory:\par \{online\_memory\}\par Search action: \{available\_or\_unavailable\}\par Visible click targets:\par \{visible\_click\_targets\}\par Previous response was invalid: \{error\_reason\}\par Return Thought and Action now:\par\medskip
\end{promptbox}
\endgroup

%% file: sections/appendix/baseline_reproduction.tex
\section{Baseline Reproduction}
\label{app:baseline-reproduction}

We reproduced Skill1, SkillRise, and Mem0 with Qwen3.5-4B under a shared bank-evaluation protocol. The trained Skill1 and SkillRise checkpoints served as curators for separate bank generation. Each bank contained 50 entries of at most 100 words, constructed from matched bundles of four trajectories. Evaluation used the frozen Qwen3.5-4B executor and \texttt{all-mpnet-base-v2} retrieval, with the top 10 entries selected once and held fixed throughout each episode. Skill1 and SkillRise results are reported as means and standard deviations across three runs.

\subsection{Skill1}

We retained Skill1's query generation, skill reranking, library updates, and distillation objective~\citep{shi2026skill1}. We trained query-generation and reranking responses while excluding task-execution responses from the policy loss. Complete trajectories still supplied outcome rewards. We used the step-40 checkpoints to generate the evaluation banks.

\subsection{SkillRise}

We retained SkillRise's cross-task training and decoupled credit assignment~\citep{yao2026skillrise}. Each trial comprised three tasks from the same family, with the skill document updated between tasks and reset between trials. Solving responses received feedback from their own task outcomes, whereas curation responses received feedback from subsequent tasks.

\subsection{Mem0}

We evaluated both semantic and procedural banks constructed with Mem0 v2.1.0~\citep{chhikara2025mem0} and Qwen3.5-4B, retaining the package's native memory-writing prompts and exporting memories from the same trajectory bundles into the shared bank format. For each evaluation seed, we averaged the results of the two bank types, then reported the mean and standard deviation across three seeds.

\subsection{Human Knowledge Construction}
\label{app:human-knowledge}

For the Human-written baseline, we sampled trajectories using the same procedure and distributed them to volunteers, who wrote knowledge entries based on these trajectories. We then assembled these entries into knowledge banks following the shared construction protocol and evaluated the resulting banks under the same evaluation protocol.

%% file: sections/appendix/experiment_details.tex
\section{Experimental Details}
\label{app:experiment-details}

\subsection{Hyperparameter Details}

Table~\ref{tab:training_hyperparameters} summarizes the training configuration for the Qwen3.5-4B curator. We update all curator parameters while keeping the executor and reference model frozen. After training, the curator generates separate knowledge banks for downstream evaluation.

\input{tables/appendix_hyper_parameter}

\subsection{Data Sources and Task Splits}
\label{app:data-splits}

Table~\ref{tab:data-splits} summarizes the task sources used
by Knowledge Weaver. The full training split is eligible
for reward sampling; each group samples eight tasks rather
than evaluating the entire split at every update.

\begin{table}[t]
\centering
\small
\renewcommand{\arraystretch}{1.15}
\begin{tabularx}{\linewidth}{
    @{}>{\raggedright\arraybackslash}p{0.25\linewidth}
    >{\raggedright\arraybackslash}X
    >{\raggedright\arraybackslash}X@{}
}
\toprule
\textbf{Data role} & \textbf{ALFWorld} & \textbf{WebShop} \\
\midrule
Training source trajectories
& Collected from \texttt{train} across all six task types.
& Collected from the full official human training split. \\

SR/MSR reward tasks
& Sampled from all \texttt{train} tasks matching the
  source context's task type.
& Sampled from all 10,587 official human training tasks. \\

Evaluation-bank source trajectories
& 120 \texttt{train} trajectories, with 20 per task type.
& 300 synthetic training-task trajectories. \\

Final evaluation tasks
& All 134 \texttt{valid\_unseen} tasks.
& A fixed subset of 100 tasks from the 500 official
  human test tasks. \\
\bottomrule
\end{tabularx}
\caption{Task sources for training, bank construction,
and final evaluation.}
\label{tab:data-splits}
\end{table}

Training trajectories, reward tasks, and bank-generation
sources may share training tasks. For ALFWorld, these
sources belong to \texttt{train}, whereas final evaluation
uses \texttt{valid\_unseen}. For WebShop, official human
training and test tasks have disjoint task IDs.
Synthetic bank-generation tasks use a separate ID
namespace.

\subsection{Training Details}
\label{app:training-details}
\label{app:optimizer-details}

We initialized the curator, executor, and reference policy from Qwen3.5-4B. Training updated all curator parameters, while the executor and reference policy remained frozen. We optimized the curator with the clipped GRPO objective using AdamW, a constant learning rate, and KL regularization against the reference policy. Table~\ref{tab:training_hyperparameters} specifies the principal settings for the main Knowledge Weaver configuration on ALFWorld and WebShop.

The training batch size and GRPO group size refer to different sampling levels. Each training batch contained 32 trajectory contexts, each comprising four trajectories from the same task type. For each context, the curator sampled eight candidate knowledge entries from the identical input, giving $32\times8=256$ candidates per batch. Execution rewards were normalized within each group of eight candidates. The value 32 therefore counts input contexts per training batch, whereas the value 8 counts candidate entries per context. The PPO minibatch count in the table specifies how this generated batch was partitioned for optimization.

For execution feedback, we used eight reward-task rollouts per candidate and condition, with matched tasks shared across the standalone, bank-only, and bank-plus-candidate conditions. Standalone success and marginal success received equal weights of $0.5$. A valid \texttt{<skill>} response received a formatting bonus of $0.02$; an invalid response received a penalty of $-0.1$ without execution or MI credit. Token-level MI credit was applied only to knowledge-body tokens, with weight $0.1$ and a cap of $0.15$ after scaling. We used within-group normalization for execution rewards and batch-level normalization for the resulting token advantages.

The training bank was initialized empty and held fixed while each batch was scored. Marginal feedback used up to three retrieved entries, and bank admission required a strictly positive marginal success gain. On ALFWorld, the update quota allowed at most one entry from each of two task types per batch during the first 20 updates, and at most one entry thereafter. On WebShop, at most one entry was admitted per batch. After training, the curator generated a separate evaluation bank of 50 entries. The training retrieval count of three is distinct from the main evaluation setting of $k=10$.

\subsection{Retrieval-Count Scaling}
\label{app:retrieval-count}

We evaluate retrieval counts $k\in\{0,\ldots,15\}$ while keeping each knowledge bank, the executor, and the benchmark task set fixed. For methods also shown in Figure~\ref{fig:knowledge-tokens}, we use the same evaluation runs and replace token count with retrieval count on the horizontal axis.

Each curve uses a fixed set of runs across all retrieval counts using three runs. Each curve retains its own no-knowledge control at $k=0$.

\begin{figure}[h]
    \centering
    \includegraphics[width=\linewidth]{figures/token_shared.png}
    \par
    \begin{minipage}[t]{0.49\linewidth}
        \centering
        \includegraphics[width=\linewidth]{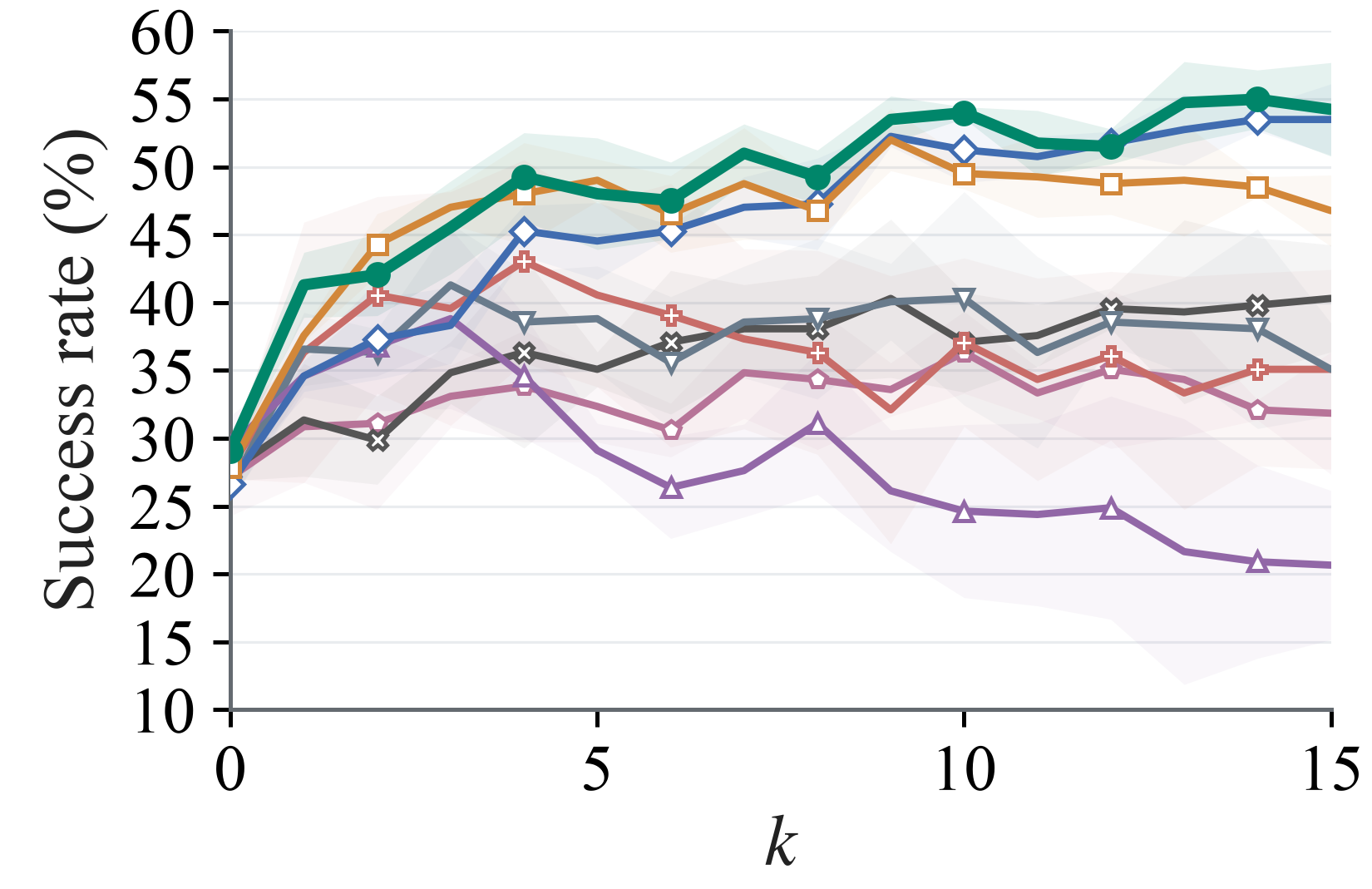}
        \par (a) ALFWorld
    \end{minipage}\hfill
    \begin{minipage}[t]{0.49\linewidth}
        \centering
        \includegraphics[width=\linewidth]{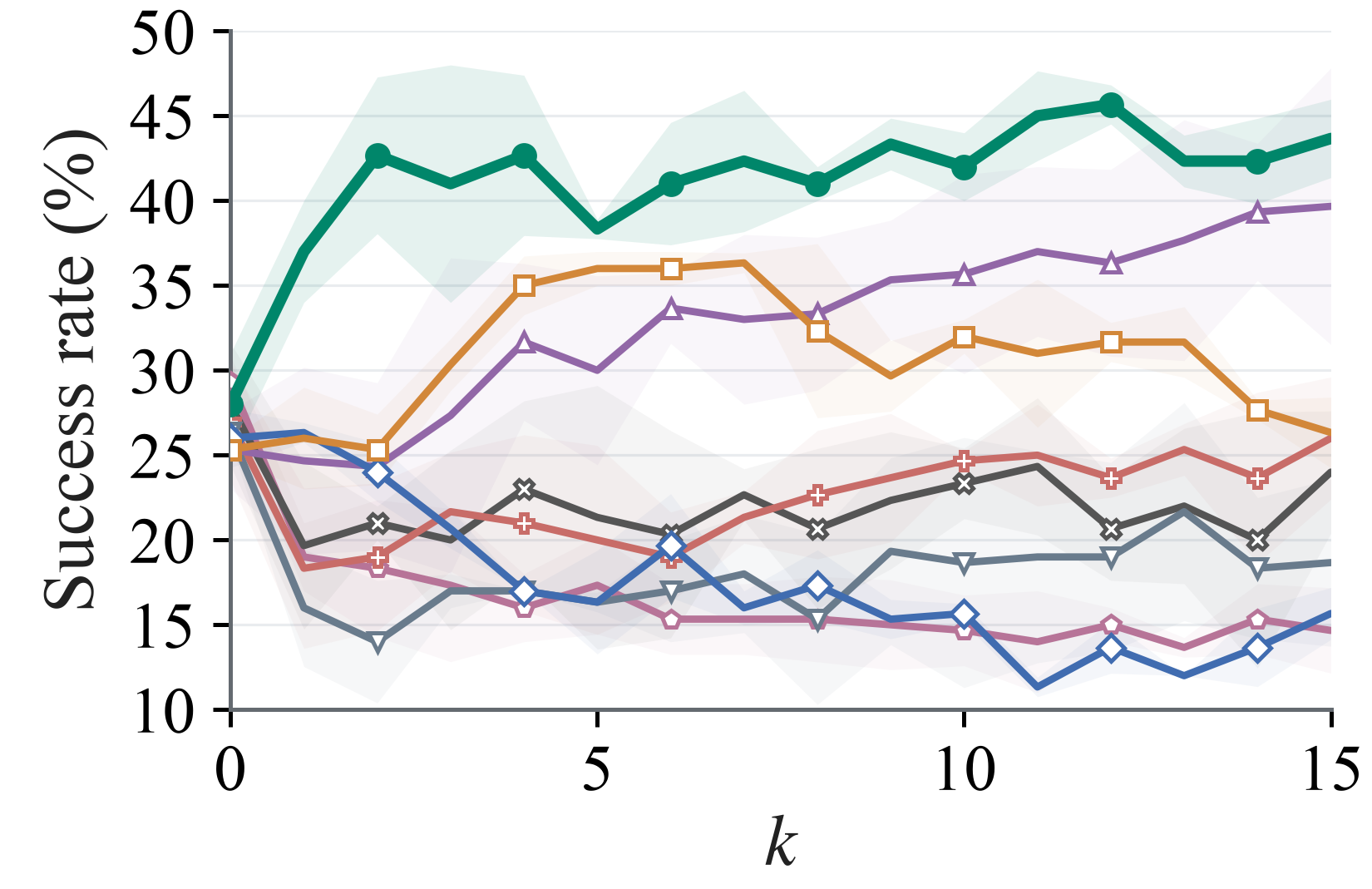}
        \par (b) WebShop
    \end{minipage}
    \caption{Success rate as the number of retrieved knowledge entries $k$ varies from 0 to 15. The horizontal axis reports the number of retrieved knowledge injected into the frozen executor. Lines and shading show the mean and one standard deviation across three runs. Each curve uses its own no-knowledge control at $k=0$.}
    \label{fig:knowledge-count}
\end{figure}

Figure~\ref{fig:knowledge-count} shows that the effect of additional retrieved knowledge varies across banks and environments. Knowledge Weaver achieves SRs of 54.0\% and 42.0\% at $k=10$ on ALFWorld and WebShop, respectively. At $k=15$, the corresponding rates are 54.2\% and 43.7\%.

%% file: tables/appendix_hyper_parameter.tex
\begin{table}[t]
\centering
\small
\setlength{\tabcolsep}{5pt}
\renewcommand{\arraystretch}{1.03}
\begin{tabular}{@{}lcc@{}}
\toprule
\textbf{Hyperparameter} & \textbf{ALFWorld} & \textbf{WebShop} \\
\midrule
\multicolumn{3}{@{}l}{\textit{Optimization}} \\
Optimizer & \multicolumn{2}{c}{AdamW} \\
Learning rate & \multicolumn{2}{c}{$1\times10^{-6}$} \\
Learning-rate schedule / warmup & \multicolumn{2}{c}{Constant / none} \\
AdamW $(\beta_1,\beta_2)$ & \multicolumn{2}{c}{$(0.9,0.999)$} \\
AdamW numerical epsilon & \multicolumn{2}{c}{$10^{-8}$} \\
Weight decay & \multicolumn{2}{c}{$0.01$} \\
Gradient-norm clipping & \multicolumn{2}{c}{$1.0$} \\
Policy-ratio clipping parameter & \multicolumn{2}{c}{$0.2$} \\
KL regularization coefficient & \multicolumn{2}{c}{$0.01$} \\
PPO minibatches per training batch & \multicolumn{2}{c}{$32$} \\
\midrule
\multicolumn{3}{@{}l}{\textit{Context Sampling and Generation}} \\
Trajectory contexts per training batch & \multicolumn{2}{c}{$32$} \\
Contexts per MI comparison set & \multicolumn{2}{c}{$8$} \\
Input trajectories per context & \multicolumn{2}{c}{$4$} \\
Candidate entries per context (GRPO group size) & \multicolumn{2}{c}{$8$} \\
Curator sampling temperature & \multicolumn{2}{c}{$1.0$} \\
Curator top-$p$ & \multicolumn{2}{c}{$1.0$} \\
Maximum generated tokens per candidate & \multicolumn{2}{c}{$256$} \\
Maximum words per knowledge entry & \multicolumn{2}{c}{$100$} \\
Maximum rendered steps per source trajectory & \multicolumn{2}{c}{$24$} \\
Source-observation character limit & $120$ & $400$ \\
\midrule
\multicolumn{3}{@{}l}{\textit{Execution Feedback and Bank Updates}} \\
Executor temperature & \multicolumn{2}{c}{$0$} \\
Maximum actions per reward rollout & $50$ & $15$ \\
Reward-task rollouts per candidate and condition & \multicolumn{2}{c}{$8$} \\
Standalone-success weight, $\lambda_{\mathrm{SR}}$ & \multicolumn{2}{c}{$0.5$} \\
Marginal-success weight, $\lambda_{\mathrm{MSR}}$ & \multicolumn{2}{c}{$0.5$} \\
Token-level MI weight, $\lambda_{\mathrm{MI}}$ & \multicolumn{2}{c}{$0.1$} \\
MI credit cap after scaling, $c_{\mathrm{MI}}$ & \multicolumn{2}{c}{$0.15$} \\
Valid-format bonus / invalid-format penalty & \multicolumn{2}{c}{$+0.02\,/\,-0.1$} \\
Retrieved entries for marginal feedback, $m$ & \multicolumn{2}{c}{$3$} \\
\bottomrule
\end{tabular}
\caption{Key training hyperparameters for the main Qwen3.5-4B Knowledge Weaver configuration. Shared values span both benchmark columns; benchmark-specific values are shown separately. Each training batch contains 32 trajectory contexts and eight candidates per context, for 256 candidate entries before optimization. Only the curator is trained. The table describes training; the separately generated evaluation banks contain 50 entries and use $k=10$ in the main comparison.}
\label{tab:training_hyperparameters}
\end{table}

%% file: sections/appendix/pseudo_code.tex
\section{Training Algorithm}
\label{app:training-algorithm}

Algorithm~\ref{alg:knowledge-weaver} summarizes Knowledge Weaver training. Candidates generated from the same trajectory context form a GRPO sampling group. Token-level MI feedback compares each candidate across source contexts while preserving its generated tokens and prefix. The executor and reference policy remain frozen throughout training.

\begin{algorithm}[t]
\caption{Knowledge Weaver curator training}
\label{alg:knowledge-weaver}
\begin{algorithmic}[1]
\Require Curator $\pi_\theta$, frozen executor, reference policy $\pi_{\mathrm{ref}}$, trajectory pool $\mathcal{T}$
\Require Updates $J$, candidates per context $G\ge2$, retrieval count $m$, admission schedule $\mathcal{Q}$
\Require Reward weights $\lambda_{\mathrm{SR}},\lambda_{\mathrm{MSR}}$, MI weight $\lambda_{\mathrm{MI}}$, MI cap $c_{\mathrm{MI}}$
\Ensure Trained curator $\pi_\theta$
\State Initialize training knowledge bank $\mathcal{B}\gets\varnothing$
\For{$j=1,\ldots,J$}
    \State $\pi_{\mathrm{old}}\gets\pi_\theta$; fix bank snapshot $\mathcal{B}^{\mathrm{snap}}\gets\mathcal{B}$
    \State Sample a batch of trajectory-context sets from $\mathcal{T}$
    \For{each context set $\mathcal{X}=(x_1,\ldots,x_N)$, with $N\ge2$}
        \For{$i=1,\ldots,N$}
            \State Sample $G$ candidates $z\sim\pi_{\mathrm{old}}(\cdot\mid x_i)$
            \State Select training reward tasks $\mathcal{D}_i$, shared by these candidates
            \State Retrieve up to $m$ entries $B_q$ from $\mathcal{B}^{\mathrm{snap}}$ for each $q\in\mathcal{D}_i$
            \For{each candidate $z$ generated from $x_i$}
                \State Initialize $U(\lbrace z\rbrace)\gets0$, $\mathrm{MSR}(z)\gets0$, and $r_t^{\mathrm{MI}}\gets0$
                \If{$z$ has valid formatting}
                    \State Evaluate $s(q,\lbrace z\rbrace)$, $s(q,B_q)$, and $s(q,B_q\cup\lbrace z\rbrace)$ on $\mathcal{D}_i$
                    \State Average standalone success to obtain $U(\lbrace z\rbrace)$
                    \State Compute $\mathrm{MSR}(z)$ using Equation~\ref{eq:marginal-success} on $\mathcal{D}_i$
                    \State Rescore $z_t$ under each $x_\ell\in\mathcal{X}$ with $z_{<t}$ fixed
                    \State Compute masked token credits $r_t^{\mathrm{MI}}$ using Equation~\ref{eq:local-mi-credit}
                \EndIf
                \State $R_{\mathrm{exec}}(z)\gets\lambda_{\mathrm{SR}}U(\lbrace z\rbrace)+\lambda_{\mathrm{MSR}}\mathrm{MSR}(z)+r_{\mathrm{fmt}}(z)$
            \EndFor
            \State Normalize $R_{\mathrm{exec}}(z)$ within the $G$ candidates generated from $x_i$
        \EndFor
    \EndFor
    \State Combine normalized execution rewards and local MI credit as $A_{g,t}$ (Equation~\ref{eq:combined-advantage})
    \State Compute $\widehat A_{g,t}$ by reference regularization and batch normalization (Equation~\ref{eq:final-advantage})
    \State Update $\theta$ using the clipped GRPO objective with $\pi_{\mathrm{old}}$ and $\widehat A_{g,t}$
    \State Select valid candidates absent from $\mathcal{B}^{\mathrm{snap}}$ with $\mathrm{MSR}(z)>0$
    \State Admit the highest-gain eligible candidates subject to $\mathcal{Q}_j$ to obtain $\mathcal{B}$
\EndFor
\State \Return $\pi_\theta$
\end{algorithmic}
\end{algorithm}

All execution conditions for a candidate use the same tasks, with $B_q$ fixed in the marginal comparison. With an empty bank, the bank-only condition reduces to execution without knowledge. Invalid candidates receive only the formatting penalty. MI credit is added after execution advantages are computed, so it is not propagated to preceding tokens through return estimation.

The admission schedule $\mathcal{Q}_j$ specifies the eligible task types and update quota. On ALFWorld, we admit at most one candidate per scheduled task type, selecting two types per update for the first 20 updates and one thereafter. On WebShop, we admit at most one candidate per update. Admission ranks candidates by $\mathrm{MSR}$ alone; if no eligible candidate has positive gain, the bank remains unchanged. The updated bank is used in the next training batch. After training, the curator generates a separate knowledge bank for downstream evaluation.

\subsection{Optimization details}
\label{app:optimizer-details}

\paragraph{Feedback and final advantages.}
The combined SR, MSR, and formatting rewards in Equation~\ref{eq:combined-advantage} are normalized across candidates from the same source context before being placed at the final generated token. We use zero value estimates, unit discount, and unit GAE parameter. Reference-policy regularization contributes a penalty at each generated token. Let $\ell_{g,t}^{\mathrm{beh}}$ and $\ell_{g,t}^{\mathrm{ref}}$ denote the sampled token log-probabilities under the behavior and frozen reference policies. The sampled KL term is
\begin{equation}
\kappa_{g,t}=\operatorname{clip}\!\left(\ell_{g,t}^{\mathrm{beh}}-\ell_{g,t}^{\mathrm{ref}},-10,10\right).
\label{eq:sampled-kl-reward}
\end{equation}
Backward return computation accumulates these penalties from token $t$ to the entry's final token $T_g$. Local MI credit is added after this computation, so it is not accumulated across token positions. Using the modified advantage $A_{g,t}$ from Equation~\ref{eq:combined-advantage}, the final advantages are
\begin{equation}
\widehat A_{g,t}=\operatorname{Norm}_{\mathrm{batch}}\!\left(A_{g,t}-\beta\sum_{s=t}^{T_g}\kappa_{g,s}\right).
\label{eq:final-advantage}
\end{equation}
The sum includes only unmasked generated tokens. Batch normalization subtracts the mean and divides by the sample standard deviation plus $\epsilon_{\mathrm{num}}$ over all such tokens in the rollout batch. Group reward normalization also uses sample standard deviation, with $\epsilon_{\mathrm{num}}=10^{-5}$ in both operations. The configurations use $\beta=0.01$ and clipping radius $\varepsilon=0.2$. Reference regularization enters through these rewards, without an additional differentiated KL term in the loss.

The full configuration uses $\lambda_{\mathrm{SR}}=\lambda_{\mathrm{MSR}}=0.5$, $\lambda_{\mathrm{MI}}=0.1$, and $c_{\mathrm{MI}}=0.15$. Formatting adds $0.02$ for valid entries and gives invalid entries a reward of $-0.1$. The backend also clips scalar rewards before group normalization and local token credits before addition to $[-20,20]$. This bound is inactive for these reward ranges and does not alter the equations above.

\paragraph{Clipped update.}
We use the clipped policy surrogate underlying GRPO \citep{shao2024deepseekmath}, averaged over valid tokens in each optimization minibatch. For token advantages $A=\{A_{g,t}\}$, the loss is
\begin{equation}
\mathcal{L}_{\mathrm{clip}}(\theta;A)
=-\frac{1}{|\mathcal{V}|}\sum_{(g,t)\in\mathcal{V}}
\min\!\left\{
\widetilde\rho_{g,t}(\theta)A_{g,t},\,
\operatorname{clip}\!\left(\widetilde\rho_{g,t}(\theta),1-\varepsilon,1+\varepsilon\right)A_{g,t}
\right\},
\label{eq:grpo-loss}
\end{equation}
where $\mathcal{V}$ contains unmasked token positions in an optimization minibatch and $g$ indexes candidates across source contexts. Training uses $A=\widehat A$ from Equation~\ref{eq:final-advantage}, with all advantage quantities held fixed during differentiation. The ratio $\widetilde\rho_{g,t}$ compares the current token probability with a detached proximal probability,
\begin{equation}
\widetilde\rho_{g,t}(\theta)=
\exp\!\left(\log\pi_\theta(z_t^{(g)}\mid x^{(g)},z_{<t}^{(g)})-\ell_{g,t}^{\mathrm{prox}}\right),
\label{eq:policy-ratio}
\end{equation}
where $x^{(g)}$ is candidate $g$'s source context. Standard GRPO uses the behavior log-probability as $\ell_{g,t}^{\mathrm{prox}}$. Our backend uses a log-linear approximation based on the policy versions recorded during generation. For current training version $v$ and token generation version $v_{g,t}^{\mathrm{beh}}<v$, define
\begin{equation}
\begin{gathered}
\alpha_{g,t}=\operatorname{clip}\!\left(\frac{v-1-v_{g,t}^{\mathrm{beh}}}{v-v_{g,t}^{\mathrm{beh}}},0,1\right),\\
\ell_{g,t}^{\mathrm{prox}}=(1-\alpha_{g,t})\ell_{g,t}^{\mathrm{beh}}+
\alpha_{g,t}\operatorname{stopgrad}\!\left(\log\pi_\theta(z_t^{(g)}\mid x^{(g)},z_{<t}^{(g)})\right).
\end{gathered}
\label{eq:proximal-logprob}
\end{equation}
Otherwise $\alpha_{g,t}=0$, and the behavior log-probability is used directly. Samples from the immediately preceding version also have $\alpha_{g,t}=0$. The denominator is detached even when its numerical value incorporates current-policy log-probabilities. We use neither dual clipping nor rejection sampling. Algorithm~\ref{alg:knowledge-weaver} omits this version bookkeeping for readability; it does not change the knowledge rewards or bank-admission criterion. All updates apply only to the curator, with the executor and reference policy frozen.

%% file: ref.bib
@article{chhikara2025mem0,
  title={Mem0: Building production-ready ai agents with scalable long-term memory},
  author={Chhikara, Prateek and Khant, Dev and Aryan, Saket and Singh, Taranjeet and Yadav, Deshraj},
  journal={arXiv preprint arXiv:2504.19413},
  year={2025}
}

@article{xu2026mem,
  title={A-mem: Agentic memory for llm agents},
  author={Xu, Wujiang and Liang, Zujie and Mei, Kai and Gao, Hang and Tan, Juntao and Zhang, Yongfeng},
  journal={Advances in Neural Information Processing Systems},
  volume={38},
  pages={17577--17604},
  year={2026}
}

@inproceedings{zhao2024expel,
  title={Expel: Llm agents are experiential learners},
  author={Zhao, Andrew and Huang, Daniel and Xu, Quentin and Lin, Matthieu and Liu, Yong-Jin and Huang, Gao},
  booktitle={Proceedings of the AAAI Conference on Artificial Intelligence},
  volume={38},
  number={17},
  pages={19632--19642},
  year={2024}
}

@inproceedings{zhang2026agentic,
  title={Agentic context engineering: Evolving contexts for self-improving language models},
  author={Zhang, Qizheng and Hu, Changran and Upasani, Shubhangi and Ma, Boyuan and Hong, Fenglu and Kamanuru, Vamsidhar and Rainton, Jay and Wu, Chen and Ji, Mengmeng and Li, Hanchen and others},
  booktitle={International Conference on Learning Representations},
  volume={2026},
  pages={86069--86100},
  year={2026}
}

@article{lu2026skill0,
  title={Skill0: In-context agentic reinforcement learning for skill internalization},
  author={Lu, Zhengxi and Yao, Zhiyuan and Wu, Jinyang and Han, Chengcheng and Gu, Qi and Cai, Xunliang and Lu, Weiming and Xiao, Jun and Zhuang, Yueting and Shen, Yongliang},
  journal={arXiv preprint arXiv:2604.02268},
  year={2026}
}

@article{zhu2026skill0,
  title={Skill0. 5: Joint Skill Internalization and Utilization for Out-of-Distribution Generalization in Agentic Reinforcement Learning},
  author={Zhu, Jiapeng and Yu, Jianxiang and Zhao, Yibo and Han, Chengcheng and Gu, Qi and Cai, Xunliang and Li, Xiang and Qian, Weining},
  journal={arXiv preprint arXiv:2605.28424},
  year={2026}
}

@article{shi2026skill1,
  title={Skill1: Unified evolution of skill-augmented agents via reinforcement learning},
  author={Shi, Yaorui and Chen, Yuxin and Lu, Zhengxi and Miao, Yuchun and Liu, Shugui and Gu, Qi and Cai, Xunliang and Wang, Xiang and Zhang, An},
  journal={arXiv preprint arXiv:2605.06130},
  year={2026}
}

@article{wang2023voyager,
  title={Voyager: An open-ended embodied agent with large language models},
  author={Wang, Guanzhi and Xie, Yuqi and Jiang, Yunfan and Mandlekar, Ajay and Xiao, Chaowei and Zhu, Yuke and Fan, Linxi and Anandkumar, Anima},
  journal={arXiv preprint arXiv:2305.16291},
  year={2023}
}

@article{shridhar2020alfworld,
  title={Alfworld: Aligning text and embodied environments for interactive learning},
  author={Shridhar, Mohit and Yuan, Xingdi and C{\^o}t{\'e}, Marc-Alexandre and Bisk, Yonatan and Trischler, Adam and Hausknecht, Matthew},
  journal={arXiv preprint arXiv:2010.03768},
  year={2020}
}

@article{yao2022webshop,
  title={Webshop: Towards scalable real-world web interaction with grounded language agents},
  author={Yao, Shunyu and Chen, Howard and Yang, John and Narasimhan, Karthik},
  journal={Advances in Neural Information Processing Systems},
  volume={35},
  pages={20744--20757},
  year={2022}
}

@inproceedings{ouyang2026reasoningbank,
  title={Reasoningbank: Scaling agent self-evolving with reasoning memory},
  author={Ouyang, Siru and Yan, Jun and Hsu, I and Chen, Yanfei and Jiang, Ke and Wang, Zifeng and Han, Rujun and Le, Long and Daruki, Samira and Tang, Xiangru and others},
  booktitle={International Conference on Learning Representations},
  volume={2026},
  pages={94327--94354},
  year={2026}
}

@article{fu2024autoguide,
  title={Autoguide: Automated generation and selection of context-aware guidelines for large language model agents},
  author={Fu, Yao and Kim, Dong-Ki and Kim, Jaekyeom and Sohn, Sungryull and Logeswaran, Lajanugen and Bae, Kyunghoon and Lee, Honglak},
  journal={Advances in Neural Information Processing Systems},
  volume={37},
  pages={119919--119948},
  year={2024}
}

@article{li2026skillsbench,
  title={SkillsBench: Benchmarking how well agent skills work across diverse tasks},
  author={Li, Xiangyi and Liu, Yimin and Chen, Wenbo and You, Bingran and Di, Zonglin and He, Yifeng and Zheng, Shenghan and Choe, Kyoung Whan and Sun, Jiankai and Wang, Shuyi and others},
  journal={arXiv preprint arXiv:2602.12670},
  year={2026}
}

@article{zheng2025skillweaver,
  title={Skillweaver: Web agents can self-improve by discovering and honing skills},
  author={Zheng, Boyuan and Fatemi, Michael Y and Jin, Xiaolong and Wang, Zora Zhiruo and Gandhi, Apurva and Song, Yueqi and Gu, Yu and Srinivasa, Jayanth and Liu, Gaowen and Neubig, Graham and others},
  journal={arXiv preprint arXiv:2504.07079},
  year={2025}
}

@article{mi2026skill,
  title={Skill-pro: Learning reusable skills from experience via non-parametric ppo for llm agents},
  author={Mi, Qirui and Ma, Zhijian and Yang, Mengyue and Li, Haoxuan and Wang, Yisen and Zhang, Haifeng and Wang, Jun},
  journal={arXiv preprint arXiv:2602.01869},
  year={2026}
}

@article{yang2026autoskill,
  title={Autoskill: Experience-driven lifelong learning via skill self-evolution},
  author={Yang, Yutao and Li, Junsong and Pan, Qianjun and Zhan, Bihao and Cai, Yuxuan and Du, Lin and Zhou, Jie and Chen, Kai and Chen, Qin and Li, Xin and others},
  journal={arXiv preprint arXiv:2603.01145},
  year={2026}
}

@article{ni2026trace2skill,
  title={Trace2skill: Distill trajectory-local lessons into transferable agent skills},
  author={Ni, Jingwei and Liu, Yihao and Liu, Xinpeng and Sun, Yutao and Zhou, Mengyu and Cheng, Pengyu and Wang, Dexin and Zhao, Erchao and Jiang, Xiaoxi and Jiang, Guanjun},
  journal={arXiv preprint arXiv:2603.25158},
  year={2026}
}

@article{xia2026skillrl,
  title={Skillrl: Evolving agents via recursive skill-augmented reinforcement learning},
  author={Xia, Peng and Chen, Jianwen and Wang, Hanyang and Liu, Jiaqi and Zeng, Kaide and Wang, Yu and Han, Siwei and Zhou, Yiyang and Zhao, Xujiang and Chen, Haifeng and others},
  journal={arXiv preprint arXiv:2602.08234},
  year={2026}
}

@inproceedings{yuan2024craft,
  title={Craft: Customizing llms by creating and retrieving from specialized toolsets},
  author={Yuan, Lifan and Chen, Yangyi and Wang, Xingyao and Fung, Yi and Peng, Hao and Ji, Heng},
  booktitle={International Conference on Learning Representations},
  volume={2024},
  pages={40097--40125},
  year={2024}
}

@article{yao2026skillrise,
  title={Skillrise: Agentic reinforcement learning for cross-task skill evolution},
  author={Yao, Zhiyuan and Chen, Yuxin and Lu, Zhengxi and Xu, Zishan and Sun, Yueqing and Guo, Yifu and Lu, Yuquan and Cai, Zhengzhou and Zhang, Kangning and Han, Zhuowen and others},
  journal={arXiv preprint arXiv:2607.26784},
  year={2026}
}

@misc{openai2025gpt5,
  author       = {{OpenAI}},
  title        = {{GPT-5 System Card}},
  year         = {2025},
  howpublished = {System card},
  url          = {https://openai.com/index/gpt-5-system-card/}
}

@misc{anthropic2026claudeopus48,
  author       = {{Anthropic}},
  title        = {{Claude Opus 4.8 System Card}},
  year         = {2026},
  howpublished = {System card},
  url          = {https://www.anthropic.com/claude-opus-4-8-system-card}
}

@misc{google2026gemini37flash,
  author       = {{Google DeepMind}},
  title        = {{Gemini 3.7 Flash Model Card}},
  year         = {2026},
  month        = aug,
  howpublished = {Model card},
  url          = {https://deepmind.google/models/model-cards/gemini-3-7-flash/}
}

@inproceedings{reimers2019sentence,
  title={Sentence-bert: Sentence embeddings using siamese bert-networks},
  author={Reimers, Nils and Gurevych, Iryna},
  booktitle={Proceedings of the 2019 conference on empirical methods in natural language processing and the 9th international joint conference on natural language processing (EMNLP-IJCNLP)},
  pages={3982--3992},
  year={2019}
}

@article{shao2024deepseekmath,
  title={Deepseekmath: Pushing the limits of mathematical reasoning in open language models},
  author={Shao, Zhihong and Wang, Peiyi and Zhu, Qihao and Xu, Runxin and Song, Junxiao and Bi, Xiao and Zhang, Haowei and Zhang, Mingchuan and Li, YK and Wu, Yang and others},
  journal={arXiv preprint arXiv:2402.03300},
  year={2024}
}

@article{shinn2023reflexion,
  title={Reflexion: Language agents with verbal reinforcement learning},
  author={Shinn, Noah and Cassano, Federico and Gopinath, Ashwin and Narasimhan, Karthik and Yao, Shunyu},
  journal={Advances in neural information processing systems},
  volume={36},
  pages={8634--8652},
  year={2023}
}

@inproceedings{li2016diversity,
  title={A diversity-promoting objective function for neural conversation models},
  author={Li, Jiwei and Galley, Michel and Brockett, Chris and Gao, Jianfeng and Dolan, William B},
  booktitle={Proceedings of the 2016 conference of the North American chapter of the association for computational linguistics: human language technologies},
  pages={110--119},
  year={2016}
}

@article{wang2024agent,
  title={Agent workflow memory},
  author={Wang, Zora Zhiruo and Mao, Jiayuan and Fried, Daniel and Neubig, Graham},
  journal={arXiv preprint arXiv:2409.07429},
  year={2024}
}

@article{majumder2023clin,
  title={Clin: A continually learning language agent for rapid task adaptation and generalization},
  author={Majumder, Bodhisattwa Prasad and Mishra, Bhavana Dalvi and Jansen, Peter and Tafjord, Oyvind and Tandon, Niket and Zhang, Li and Callison-Burch, Chris and Clark, Peter},
  journal={arXiv preprint arXiv:2310.10134},
  year={2023}
}

@article{chen2024automanual,
  title={Automanual: Constructing instruction manuals by llm agents via interactive environmental learning},
  author={Chen, Minghao and Li, Yihang and Yang, Yanting and Yu, Shiyu and Lin, Binbin and He, Xiaofei},
  journal={Advances in Neural Information Processing Systems},
  volume={37},
  pages={589--631},
  year={2024}
}

@inproceedings{suzgun2026dynamic,
  title={Dynamic cheatsheet: Test-time learning with adaptive memory},
  author={Suzgun, Mirac and Yuksekgonul, Mert and Bianchi, Federico and Jurafsky, Dan and Zou, James},
  booktitle={Proceedings of the 19th Conference of the European Chapter of the Association for Computational Linguistics (Volume 1: Long Papers)},
  pages={7080--7106},
  year={2026}
}

@article{yao2022react,
  title={React: Synergizing reasoning and acting in language models},
  author={Yao, Shunyu and Zhao, Jeffrey and Yu, Dian and Du, Nan and Shafran, Izhak and Narasimhan, Karthik and Cao, Yuan},
  journal={arXiv preprint arXiv:2210.03629},
  year={2022}
}

@inproceedings{gur2024real,
  title={A real-world webagent with planning, long context understanding, and program synthesis},
  author={Gur, Izzeddin and Furuta, Hiroki and Huang, Austin and Safdari, Mustafa and Matsuo, Yutaka and Eck, Douglas and Faust, Aleksandra},
  booktitle={International Conference on Learning Representations},
  volume={2024},
  pages={52690--52717},
  year={2024}
}

@article{yang2024swe,
  title={Swe-agent: Agent-computer interfaces enable automated software engineering},
  author={Yang, John and Jimenez, Carlos and Wettig, Alexander and Lieret, Kilian and Yao, Shunyu and Narasimhan, Karthik and Press, Ofir},
  journal={Advances in Neural Information Processing Systems},
  volume={37},
  pages={50528--50652},
  year={2024}
}

@inproceedings{zheng2024synapse,
  title={Synapse: Trajectory-as-exemplar prompting with memory for computer control},
  author={Zheng, Longtao and Wang, Rundong and Wang, Xinrun and An, Bo},
  booktitle={International Conference on Learning Representations},
  volume={2024},
  pages={19036--19066},
  year={2024}
}

@article{wang2026ragen,
  title={Ragen-2: Reasoning collapse in agentic rl},
  author={Wang, Zihan and Gui, Chi and Jin, Xing and Wang, Qineng and Liu, Licheng and Wang, Kangrui and Chen, Shiqi and Li, Linjie and Yang, Zhengyuan and Zhang, Pingyue and others},
  journal={arXiv preprint arXiv:2604.06268},
  year={2026}
}

@book{cover1991elements,
  title={Elements of information theory},
  author={Cover, Thomas M and Thomas, Joy A and Kieffer, John},
  volume={2},
  year={1991},
  publisher={wiley New York}
}

@article{zhang2018generating,
  title={Generating informative and diverse conversational responses via adversarial information maximization},
  author={Zhang, Yizhe and Galley, Michel and Gao, Jianfeng and Gan, Zhe and Li, Xiujun and Brockett, Chris and Dolan, Bill},
  journal={Advances in Neural Information Processing Systems},
  volume={31},
  year={2018}
}

@inproceedings{pan2020adversarial,
  title={Adversarial mutual information for text generation},
  author={Pan, Boyuan and Yang, Yazheng and Liang, Kaizhao and Kailkhura, Bhavya and Jin, Zhongming and Hua, Xian-Sheng and Cai, Deng and Li, Bo},
  booktitle={International Conference on Machine Learning},
  pages={7476--7486},
  year={2020},
  organization={PMLR}
}

@article{eysenbach2018diversity,
  title={Diversity is all you need: Learning skills without a reward function},
  author={Eysenbach, Benjamin and Gupta, Abhishek and Ibarz, Julian and Levine, Sergey},
  journal={arXiv preprint arXiv:1802.06070},
  year={2018}
}

@article{oord2018representation,
  title={Representation learning with contrastive predictive coding},
  author={Oord, Aaron van den and Li, Yazhe and Vinyals, Oriol},
  journal={arXiv preprint arXiv:1807.03748},
  year={2018}
}

@article{hjelm2018learning,
  title={Learning deep representations by mutual information estimation and maximization},
  author={Hjelm, R Devon and Fedorov, Alex and Lavoie-Marchildon, Samuel and Grewal, Karan and Bachman, Phil and Trischler, Adam and Bengio, Yoshua},
  journal={arXiv preprint arXiv:1808.06670},
  year={2018}
}

@article{chen2016infogan,
  title={Infogan: Interpretable representation learning by information maximizing generative adversarial nets},
  author={Chen, Xi and Duan, Yan and Houthooft, Rein and Schulman, John and Sutskever, Ilya and Abbeel, Pieter},
  journal={Advances in neural information processing systems},
  volume={29},
  year={2016}
}

@article{carl,
  title={Carl: Critical action focused reinforcement learning for multi-step agent},
  author={Shen, Leyang and Zhang, Yang and Ling, Chun Kai and Zhao, Xiaoyan and Chua, Tat-Seng},
  journal={arXiv e-prints},
  pages={arXiv--2512},
  year={2025}
}

@article{zhang2026nextmem,
  title={Nextmem: Towards latent factual memory for llm-based agents},
  author={Zhang, Zeyu and Li, Rui and Zhao, Xiaoyan and Zhang, Yang and Wang, Wenjie and Chen, Xu and Chua, Tat-Seng},
  journal={arXiv preprint arXiv:2603.15634},
  year={2026}
}
